\documentclass[letterpaper]{article} %
\usepackage{aaai24}  %
\usepackage{times}  %
\usepackage{helvet}  %
\usepackage{courier}  %
\usepackage[hyphens]{url}  %
\usepackage{graphicx} %
\usepackage{natbib}  %
\usepackage{caption} %
\usepackage{algorithm}

\usepackage{newfloat}
\usepackage{listings}
\DeclareCaptionStyle{ruled}{labelfont=normalfont,labelsep=colon,strut=off} %
\floatstyle{ruled}
\newfloat{listing}{tb}{lst}{}
\floatname{listing}{Listing}
\usepackage[utf8]{inputenc}
\usepackage[inline]{enumitem}
\setlist[enumerate]{label=(\arabic*)}
\usepackage{amsmath}
\usepackage{amsfonts}
\usepackage{centernot}
\usepackage{subcaption}
\usepackage{graphicx}
\usepackage{verbatim}
\usepackage{booktabs}
\usepackage[capitalise,noabbrev,nameinlink]{cleveref}
\usepackage{natbib}
\usepackage{nicefrac}

\usepackage{caption}

\usepackage[noend]{algpseudocode}
\algnewcommand{\LineComment}[1]{\State \(\triangleright\) #1}

\usepackage{xspace}
\usepackage{amssymb}

\newcommand{\gitLink}{\url{github.com/lava-lab/RATM}}

\DeclareMathOperator{\E}{\mathbb{E}}

\DeclareMathOperator*{\arginf}{arg\,inf}

\DeclareMathOperator{\mv}{MV}

\newcommand{\RACNO}{\text{RAM-MDP}\xspace}
\newcommand{\RACNOs}{\text{RAM-MDPs}\xspace}
\newcommand{\RAM}{\text{RAM-MDP}\xspace}
\newcommand{\RAMs}{\text{RAM-MDPs}\xspace}
\newcommand{\mvr}{\text{MV}_{\text{RATM}}\xspace}

\newcommand{\ATM}{{\text{ATM}}\xspace}
\newcommand{\uATM}{{\text{RATM}}\xspace}
\newcommand{\ucATM}{{\text{MLATM}}\xspace}
\newcommand{\ucATMpes}{{$\ucATM\emph{-pes}$}\xspace}
\newcommand{\ucATMopt}{{$\ucATM\emph{-opt}$}\xspace}
\newcommand{\ucATMavg}{{$\ucATM\emph{-avg}$}\xspace}
\newcommand{\ATMavg}{{$\ATM\emph{-avg}$}\xspace}
\newcommand{\ATMpes}{{$\ATM\emph{-pes}$}\xspace}

\newcommand{\CR}{ {\text{ML}} \xspace}
\newcommand{\CRMDP}{ {\text{CRMDP}} \xspace}

\newcommand{\env}[1]{\textsc{#1}}
\newcommand{\uMVenv}{\env{lucky-unlucky}\xspace}
\newcommand{\uMVtwoenv}{\env{a-b}\xspace}
\newcommand{\drone}{\env{drone}\xspace}
\newcommand{\snakemaze}{\env{snakemaze}\xspace}

\usepackage{listings}

\usepackage{amsthm}
\newtheorem{theorem}{Theorem}

\newtheorem*{lemma*}{Lemma}

\crefname{observation}{Observation}{Observations}
\newtheorem{corollary}{Corollary}
\newtheorem*{corollary*}{Corollary}
\crefname{corollary}{Corollary}{Corollaries}
\theoremstyle{definition}
\newtheorem{definition}{Definition}

\crefname{finding}{Result}{Results}

\usepackage[most]{tcolorbox}
\newtcolorbox{leftvrule}[1][]{colback=white,  boxrule=0pt, boxsep=0pt, breakable, enhanced jigsaw, borderline west={1.5pt}{0pt}{black},
before skip=5pt,after skip=5pt,
#1}

\usepackage{tikz}
\usetikzlibrary{arrows.meta, shapes, shapes.geometric, arrows, fit, calc, positioning, automata, backgrounds, shadows}
\tikzset{elliptic state/.style={draw,ellipse}}
\usepackage{xcolor}

\definecolor{InkBlue}{RGB}{166,206,227} %
\definecolor{InkGreen}{RGB}{178,223,138} %
\definecolor{InkRed}{RGB}{251,154,153} %
\definecolor{InkYellow}{RGB}{253,191,111} %

\definecolor{InkGrey}{HTML}{bcc2c8} %

\tikzset{tightR/.style={diamond, inner sep=0cm, text width=0.7cm, align=center, draw=black, fill=white} }
\tikzset{tightS/.style={circle, inner sep=0cm, text width=0.7cm, align=center, draw=black, fill=white } }
\tikzset{tightBox/.style={rectangle, inner sep=0.1cm, align=center, draw=black, fill=white } }
\tikzset{tightP/.style={shape=diamond, minimum size=0.5cm, fill=black}}
\tikzset{help/.style={draw=none}}

\title{Robust Active Measuring under Model Uncertainty}
\author {
    Merlijn Krale\textsuperscript{\rm 1},
    Thiago D. Simão\textsuperscript{\rm 2},
    Jana Tumova\textsuperscript{\rm 3},
    Nils Jansen\textsuperscript{\rm 1,4}    
}
\affiliations {
    \textsuperscript{\rm 1} Radboud University Nijmegen, the Netherlands,  
    \textsuperscript{\rm 2} Eindhoven University of Technology, the Netherlands,\\
    \textsuperscript{\rm 3} KTH Royal Institute of Technology, Stockholm, Sweden,
    \textsuperscript{\rm 4} Ruhr-University Bochum, Germany\\    
    merlijn.krale@ru.nl, t.simao@tue.nl, tumova@kth.se, nils.jansen@ru.nl
}

\begin{document}

\maketitle

\begin{abstract}
    Partial observability and uncertainty are common problems in sequential decision-making that particularly impede the use of formal models such as Markov decision processes (MDPs).
    However, in practice, agents may be able to employ costly sensors to \emph{measure} their environment and resolve partial observability by gathering information. 
    Moreover, imprecise transition functions can capture model uncertainty.
    We combine these concepts and extend MDPs to \emph{robust active-measuring MDPs (RAM-MDPs)}.
    We present an active-measure heuristic to solve RAM-MDPs efficiently and show that model uncertainty can, counterintuitively, let agent take fewer measurements.
    We propose a method to counteract this behavior while only incurring a bounded additional cost.
    We empirically compare our methods to several baselines and show their superior scalability and performance.
\end{abstract}

\section{Introduction}

Markov decision processes \citep[MDPs;][]{DBLP:books/wi/Puterman94} are a standard sequential decision-making model \citep{DBLP:journals/robotics/KormushevCC13,DBLP:journals/comsur/LeiTZLZS20,DBLP:journals/tits/SunbergK22}.
However, in MDPs, the decision-maker has full knowledge of the dynamics of the environment and its current state, which is often unrealistic.
Well-studied frameworks exist to relax these assumptions.
To represent model uncertainty, \emph{robust MDPs} \citep[RMDPs;][]{DBLP:journals/ior/NilimG05} extend MDPs by replacing transition probabilities with uncertainty sets.
To represent state uncertainty, \emph{partially observable MDPs} \citep[POMDPs;][]{DBLP:journals/ai/KaelblingLC98} extend MDPs with an observation function, which dictates how the agent gains information while interacting with the environment.

\emph{Active-measure} MDPs are a subset of the latter, where agents have direct control over when and how they gather information, which has an associated cost~\citep{DBLP:conf/ai/BellingerC0T21}.
For example, a drone may request information from a motion capture system which has costs related to communication (\cref{fig:DroneExample}).
Furthermore, this model can capture applications in predictive maintenance and healthcare, such as diagnostics~\citep{jimenez2022deterioration, DBLP:journals/csur/YuLNY23}.
In these applications, the cost or risk of gaining information needs to be weighed against the value of obtaining more information.

\begin{figure}[tb]
    \centering
    \includegraphics[width=0.6\columnwidth]{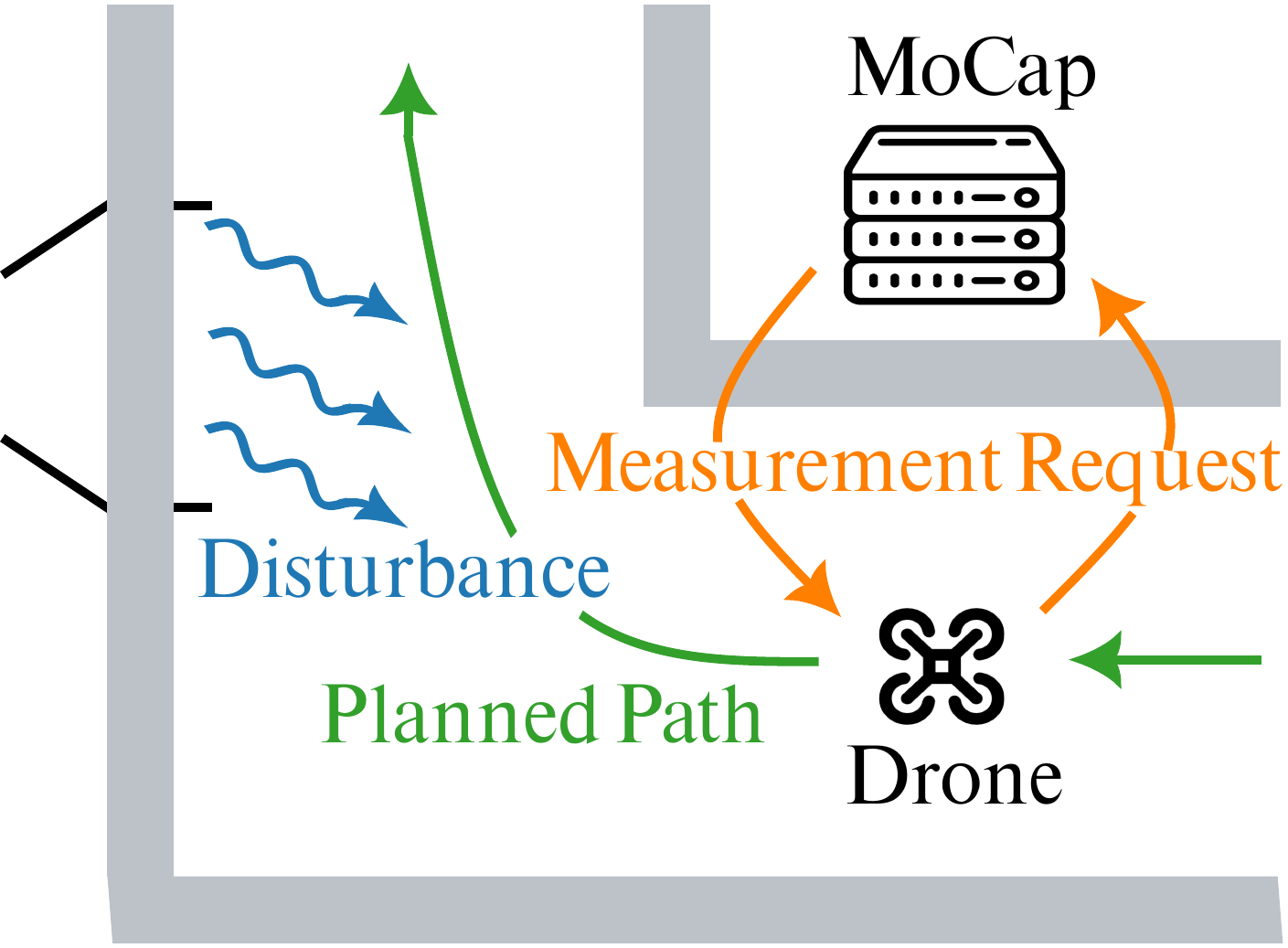}
    \caption{A motivating example.
    A drone has to plan a path through a corridor where (wind) disturbances introduce uncertainty in its position.
    An external Motion Capture (MoCap) system can provide the drone's exact position, but this uses some of its limited bandwidth.
    How should the risk of a collision be weighed against the cost of using this system?
    }
    \label{fig:DroneExample}
\end{figure}

Settings with both model and state uncertainty can be expressed as \emph{robust POMDPs} \citep[RPOMDPs;][]{@osogamiRobustPartiallyObservable2015}.
However, even though uncertain and partially observable settings have been studied extensively on their own, research on RPOMDPs has been limited in part due to their complexity.
Existing strategies for solving RPOMDPs are either exact but computationally expensive \citep{@osogamiRobustPartiallyObservable2015,@rasouliRobustPartiallyObservable2018}, or only consider policies with limited memory \citep{DBLP:conf/ijcai/Suilen0CT20,DBLP:conf/aaai/Cubuktepe0JMST21}.

Aiming to achieve better performance and scalability, this paper focuses on a subset of RPOMDPs with \emph{active measuring}, which we formally define as \emph{robust active-measuring MDPs} (\RAMs).
We make the counter-intuitive observation that high model uncertainty may discourage measuring in certain environments.
For solving a specific subset of \RAMs, we adopt a heuristic called \emph{act-then-measure} \citep[ATM;][]{DBLP:conf/aips/KraleS023} for standard active-measure environments in an uncertain setting.
This heuristic suggests partially ignoring future state uncertainty, which drastically decreases policy computation times.

Next, we propose \emph{measurement leniency}, a strategy to encourage measuring in settings with high model uncertainty.
This strategy allows the agent to make additional measurements when this would yield better results under a less pessimistic model.
We formalize this idea and prove that measurement lenient policies have a bounded lost return as compared to their fully robust counterpart.

We empirically compare both regular and measurement lenient variants of our algorithm on a number of environments.
Against a number of baselines, we demonstrate (1) the computational tractability of our method and (2) an increased robustness of policies.

\paragraph{Contributions.}
The main contributions of this work are:
\begin{enumerate*}
    \item defining \RAMs to represent active measuring in uncertain environments;
    \item analyzing the influence of model uncertainty on measuring behavior; 
    \item showing how the \emph{act-then-measure} heuristic can be used to efficiently solve a subset of \RAMs; and
    \item defining the \emph{measurement leniency} strategy to get better performance %
    in settings with high model uncertainty.
\end{enumerate*}

\section{Setting: \RAMs}
\label{sec:uACNO-MDPs}

\begin{figure}[t]
    \centering
    \begin{minipage}{.95\columnwidth}
        \centering\resizebox{1\linewidth}{!}{\pgfdeclarelayer{background}
\pgfdeclarelayer{foreground}
\pgfsetlayers{background,main,foreground}

\tikzstyle{quantity}=[draw, fill=blue!20, text width=5em, 
    text centered, minimum height=2.5em,drop shadow]
\tikzstyle{ann} = [above, text width=5em, text centered]
\tikzstyle{wa} = [sensor, text width=10em, fill=red!20, 
    minimum height=6em, rounded corners, drop shadow]
\tikzstyle{sc} = [sensor, text width=13em, fill=red!20, 
    minimum height=10em, rounded corners, drop shadow]

\def\dist{1.5cm}

\begin{tikzpicture}
    \node [tightS, fill=InkGreen] (pi) {$\pi$};
    \node [tightS, fill=InkGreen] (b) [right=0.25cm of pi] {$b_t$};

    \node [tightBox, fill=InkGrey] (N) [right=\dist + 0.5cm of b] {Nature};
    \node [tightBox, fill=InkGrey] (P) [right= 0.5cm of N] {$P {\in} \mathcal{P}$};
    \node [tightS, fill=InkGreen]  (s) [right=0.5cm of P] {$s_t$};
    
    \node [coordinate] (pib) [above right= 0.5cm and 0.25cm of pi] {}; %
    \node [coordinate] (bdown) [below=0.5 of b] {};
    \node [coordinate] (bdownright) [below right=0.6cm and 1cm of bdown] {};
    
    \node [coordinate] (Nup) [above=1.25cm of N, xshift = -0.75cm] {};
    \node [coordinate] (Nleftup) [left=1cm of Nup] {};
    \node [coordinate] (sdown) [below=0.5cm of s] {};
    \node [coordinate] (sdownleft) [below left=0.6cm and 1cm of sdown] {};

    \begin{scope}[on background layer]
        \node [shape=rectangle, minimum height=2cm, minimum width=2cm, fill=InkGrey, draw=black, xshift=0.5cm, ] (agent) {};
        \node [shape=ellipse, minimum height=2.5cm, minimum width=4.75cm, fill=InkBlue, draw=black] (env) [right= \dist of agent] {};
        \node [help] (envlabel) [above=0cm of env]  {RAM-MDP};
        \node [help] (alabel) at (agent |- envlabel) {Agent};
    \end{scope}

    \draw (pi) to[out=90, in=180] (pib) to[out=0, in=90]  (b);
    \draw (pib) to[out=90, in=180] (Nleftup) to[out=0, in=180] node[below, pos=0.2] {$\tilde{a}_t{=}(a_t,m_t)$} (Nup);
    \draw [->](Nup) to[out=0, in=90, ->] node [left, pos=0.85] {$\tilde{a}_t$} (N);
    \draw [->](Nup) to[out=0, in=90] node [left, pos=0.90, ->] {$a_t$}  (P);
    \draw [->](s) to[out=270, in = 90] (sdown) to[out=270, in=0] (sdownleft) to[out=180, in=0] node [above, pos=0.85] {$o_t {\sim} O(s_{t+1},m_t)$} node [below, pos=0.85] {$\tilde{r}_t {=} R(s_t,a_t) {-} C(m_t)$} (bdownright) to[out=180, in=270] (bdown);

    \path[->]   (N) edge  (P)
                (P) edge [bend left=50] node [above] {$s_{t{+}1}$} (s)
                (s) edge [bend left=50] node [below] {$s_t$} (P)
                (b) edge [densely dashed] (N)
                (b) edge [loop below] node [left, pos=0.75] {$b_{t{+}1}$} (b);
            ;

\end{tikzpicture}}
    \end{minipage}
    \caption{
    A visualization of agent-environment interactions in \RAMs, as explained in \cref{sec:uACNO-MDPs}.
    }
    \label{fig:uACNOInteractions}
\end{figure}
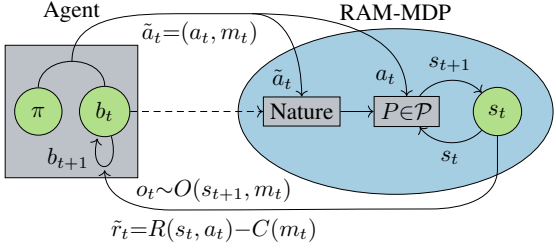
In this section, we formally define \RAMs as the combination of the RMDP \citep{DBLP:journals/ior/NilimG05} and active-measure \citep{DBLP:conf/ai/BellingerC0T21} frameworks:

\begin{definition}
    A \emph{robust active-measure MDP} (\RAM) is a tuple $\mathcal{M} {=} (S, R, \gamma, \tilde{A} {=} A {\times} M, \mathcal{P}, \mathcal{I}, O, \Omega, C)$, with state space $S$, reward function $R \colon S \times A \rightarrow \mathbb{R}$, and discount factor $\gamma$.
    $\tilde{A}$ is the set of \emph{actions}, which consists of pairs of \emph{control} and \emph{measurement} actions $\tilde{a} {=} \langle a, m \rangle {\in} A {\times} M$.
    Control actions affect the environment, while measurement actions affect what information agents gain about their current state.
    The dynamics are given by an \emph{uncertain transition function} $\mathcal{P} \colon S {\times} A {\times} S {\rightarrow} \mathcal{I}$, which gives an interval from the \emph{interval set} $[p_\text{min}, p_\text{max}] \in \mathcal{I}$ (with $ 0 {\leq} p_\text{min} {\leq} p_\text{max} {\leq} 1$) for each transition.
    Like POMDPs, \RAMs have an \emph{observation function} $O \colon S {\times} M {\times} \Omega {\rightarrow} \mathbb{R}$ with $\Omega$ the observation space.
    Lastly, the cost of measuring is given by $C \colon M {\rightarrow} \mathbb{R}$.
\end{definition}

\RAMs are a subset of RPOMDPs, with the additional property that the action space is factorized into control- and measuring actions, and $O$ and $\mathcal{P}$ are independent of these respective action types.
Furthermore, \RAMs collapse to RMDPs if all measurements have cost $0$ and a unique observation for all states, and to MDPs if, in addition, $p_\text{min}{=} p_\text{max}$
for all intervals in $\mathcal{I}$, and $\mathcal{P}$ forms a valid probability distribution for all state-action pairs.

Inspired by \citet{DBLP:conf/nips/NamFB21}, we assume that measurements are \emph{complete} and \emph{noiseless}.
Intuitively, this means agents only have two measurement options: they either take a measurement that returns full state information, or they take no measurement.
In this case, the observation set has the form $\Omega \colon S {\cup} \{\bot\}$, and the observation function is deterministic, with $O \colon S {\times} M {\rightarrow} \Omega$, such that $\forall s \colon O(\bot | s,0) {=} 1$ and $O(s | s,1) {=} 1$.
Lastly, we assume $C(0) {=} 0$ and denote $C(1) {=} c$.

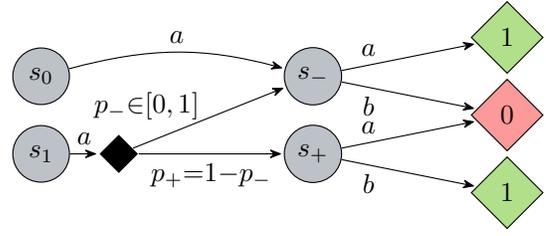
\begin{figure}[t]
    \centering
    \begin{minipage}{.85\columnwidth}
        \resizebox{1\linewidth}{!}{\begin{tikzpicture}[
    shorten >=1pt,
    node distance=1cm and 3.5cm,
    on grid,
    >={Stealth[round]},
    every state/.style={fill=InkBlue}
    ]
    \node[tightS, fill=InkGrey]  (s_0)                                         {$s_0$};
    \node[tightS, fill=InkGrey]          (s_1) [below=of s_0]     {$s_1$};

    \node[tightS, fill=InkGrey]          (s_-) [right= of s_0]                       {$s_{-}$};
    \node[tightP]    (p1)  [right= 1cm of s_1] {};
    \node[tightS, fill=InkGrey]        (s_+) [right = of s_1]                      {$s_{+}$};
    \node[tightR, fill=InkGreen ]        (R+1) [above right = 0.5cm and 2.5cm of s_-]  {$1$};
    \node[tightR, fill=InkRed]           (R-)  [below right = 0.5cm and 2.5cm of s_-]  {$0$};
    \node[tightR, fill=InkGreen ]        (R+2) [below right = 0.5cm and 2.5cm of s_+]  {$1$};

  \path[->] (s_0)   edge [above, bend left=15]               node {$a$}      (s_-)
            (s_1)   edge [above]               node {$a$}      (p1)
            (p1)    edge [below]               node {$p_{+}{=}1{-}p_{-}$} (s_+)
                    edge []          node [left, pos=0.7, xshift=-0.4cm] {$p_{-} {\in} [0,1]$} (s_-)
            (s_-)   edge [above]               node [pos=0.2] {$a$} (R+1)
            (s_-)   edge [below]               node [pos=0.2] {$b$} (R-)
            (s_+)   edge [above]               node [pos=0.2] {$a$} (R-)
            (s_+)   edge [below]               node [pos=0.2] {$b$} (R+2)
            ;
\end{tikzpicture}}
    \end{minipage}
    \caption{For initial belief $b$ over $s_0$ and $s_1$ in this RPOMDP, the expected return is minimized if, for the next belief $b'$, probabilities of being in state $s_-$ and $s_+$ are equal.
    The corresponding values of $p_-$ and $p_+$ depend on $b(s_0)$ and $b(s_1)$, thus, the worst-case transition function is belief-dependent.
    }
    \label{fig:PrmdpFail}
\end{figure}
Agent-environment interactions for \RAMs can be viewed as a two-player game between the agent and `nature', as visualized in \cref{fig:uACNOInteractions}.
Starting from an initial state $s_0$, for each time-step $t$ the agent chooses an action pair $\tilde{a}_t {=} \langle a_t, m_t \rangle$ to execute according to some policy~$\pi$.
Based on the chosen action pair, the current state, and the agent's current belief, nature picks a valid probability function $P(\cdot|s_t,\tilde{a}_t)$ from the uncertainty set, i.e., subject to the constraint that $\forall s'\colon P(s'|s_t, \tilde{a}_t) {\in} \mathcal{P}(s_t,a_t,s')$ and $\sum_{s' \in S} P(s'|s_t,\tilde{a}_t) {=} 1$.
Then, the environment transitions to a new state $s_{t+1} {\sim} P(\cdot|s_t,a_t)$, and returns a scalarized reward $\tilde{r}_t {=} R(s_t,a_t) {-} C(m_t)$ and observation $o_{t} {\sim} O(\cdot | s_{t+1}, m_t)$ to the agent.
The goal of the agent is to compute a policy $\pi$ with the highest expected discounted scalarized return.
We assume these policies are belief-based, that is, $\pi\colon b \rightarrow \tilde{A}$ for a belief $b\in \Delta(S)$ over states.

To make this problem more tractable, we make a few assumptions.
Our description of the agent-environment interactions already assumes full observability for nature, as well as a dynamic \citep{DBLP:journals/ior/NilimG05} and $(s,a)$-rectangular \citep{DBLP:journals/mor/WiesemannKR13} system.
These assumptions mean that the transition probabilities picked by nature may be different at each timestep and are independent of other transition probabilities, which are both common assumptions in RPOMDP literature.
Next, we assume nature is \emph{adversarial}, meaning it chooses transition functions to minimize the expected discounted scalarized return of the agent.
As for RPOMDPs, these assumptions mean worst-case transition probabilities are generally \emph{belief-dependent}, as shown by the RPOMDP in \cref{fig:PrmdpFail}.
Thus, the worst-case transition- and value functions, $P_\text{R}$ and $V_\text{R}$, are both belief-dependent.
We first introduce $b_\text{R}'(b,\tilde{a})$ as the expected distribution over states in the next step when taking action pair $\tilde{a} {\in} \tilde{A}$ in belief $b$:
\begin{equation}
\label{eq:robustBeliefUpdate}
    b_\text{R}'(s'|b, \tilde{a}) = \sum_{s \in S} b(s) P_\text{R}(s'|s,b,\tilde{a}).
\end{equation}
Using this notation, we define $P_\text{R}$ as follows:
\begin{equation}
\label{eq:PR_uACNO}
\begin{aligned}
    P_\text{R} (s'|s,b, \langle a,m \rangle) &= \arginf_{P_\text{R} \in \mathcal{P}(s,a,\cdot)} V_\text{R}\big(b'_\text{R}(b,\langle a,m \rangle)\big),
\end{aligned}
\end{equation}
where we note that the minimization affects both $V_\text{R}$ directly, as well as $b_\text{R}'$.
With this, $V_\text{R}$ is given as follows:
\begin{equation}
\label{eq:VR_uACNO}
\begin{aligned}
    V_\text{R}(b) =  & \max_{ \tilde{a} = \langle a, m \rangle \in  \tilde{A}} \sum_{s \in S} b(s) 
    \Big(R(s, a) -C(m) \\ 
    & + \gamma \sum_{s' \in s} P_\text{R}\big(s'|s,b,\tilde{a}\big)  V_\text{R}\big(b'_\text{R}(b,\tilde{a})\big) \Big).
\end{aligned}
\end{equation}

\section{\RAM Properties}\label{sec:ram-properties}
In this section, we highlight and discuss a number of interesting properties of \RAMs.

\begin{figure}
\renewcommand{\thesubfigure}{\lr{subfigure}}
    \hfill
    \begin{minipage}{.49\columnwidth}
        \resizebox{1\linewidth}{!}{\begin{tikzpicture}[
    shorten >=1pt,
    node distance=0.5cm and 2.75cm,
    on grid,
    >={Stealth[round]},
    every state/.style={fill=InkBlue}
    ]

    \def\distR{1.5cm}
    \node[tightS, fill=InkGrey]  (s_0)                                         {$s_0$};
    \node[tightP]    (p0)  [right= 1cm of s_0] {};
    \node[tightS, fill=InkGrey]          (s_-) [above right = of s_0]                      {$s_{-}$};
    \node[tightS, fill=InkGrey]          (s_+) [below right = of s_0]                      {$s_{+}$};
    
    \node[tightR, fill=InkYellow, ]        (R-) [above right = 0.5cm and \distR of s_-]  {$0.8$};
    \node[tightR, fill=InkRed]           (R5)  [below right = 0.5cm and \distR of s_-]  {$0$};
    \node[tightR, fill=InkGreen ]        (R+) [below right = 0.5cm and \distR of s_+]  {$1$};

  \path[->] (s_0)   edge node [above] {$a$}      (p0)
            (p0)    edge node [above, yshift=0.1cm, xshift=-0.2cm]  {$p {\in} [0,p_{\text{max}}]$} (s_-)
            (p0)    edge node [below, xshift=-0.2cm]  {$1{-}p$} (s_+)
            (s_-)   edge node [above, pos=0.15] {$a$} (R-)
            (s_-)   edge node [below, pos=0.15] {$b$} (R5)
            (s_+)   edge node [above, pos=0.15] {$a$} (R5)
            (s_+)   edge node [below, pos=0.15] {$b$} (R+)
            ;
\end{tikzpicture}}
    \end{minipage}
    \begin{minipage}{.49\columnwidth}
        \resizebox{1\linewidth}{!}{\begin{tikzpicture}[
    shorten >=1pt,
    node distance=0.5cm and 2.75cm,
    on grid,
    >={Stealth[round]},
    every state/.style={fill=InkBlue}
    ]

    \def\distR{1.5cm}
    \node[tightS, fill=InkGrey]  (s_0)                                         {$s_0$};
    \node[tightP]    (p0)  [right= 1cm of s_0] {};
    \node[tightS, fill=InkGrey]          (s_-) [above right = of s_0]                      {$s_{-}$};
    \node[tightS, fill=InkGrey]          (s_+) [below right = of s_0]                      {$s_{+}$};
    
    \node[tightR, shape=diamond, fill=InkRed, inner sep=0 ]        (R-) [above right = 0.5cm and \distR of s_-]  {${-}\infty$};
    \node[tightR, fill=InkYellow]           (R5)  [below right = 0.5cm and \distR of s_-]  {$0$};
    \node[tightR, shape=diamond, fill=InkGreen ]        (R+) [below right = 0.5cm and \distR of s_+]  {$\infty$};

  \path[->] (s_0)   edge node [above] {$a$}      (p0)
            (p0)    edge node [above, yshift=0.1cm, xshift=-0.2cm]  {$p {\in} [0,p_{\text{max}}]$} (s_-)
            (p0)    edge node [below, xshift=-0.2cm]  {$1{-}p$} (s_+)
            (s_-)   edge node [above] {$a$} (R-)
            (s_-)   edge node [below] { } (R5)
            (s_+)   edge node [above] {$b$} (R5)
            (s_+)   edge node [below] {$a$} (R+)
            ;
\end{tikzpicture}}
    \end{minipage}
    \caption{Environments \uMVtwoenv~(left) and \uMVenv~(right).}
    \label{fig:uMV}
\end{figure}
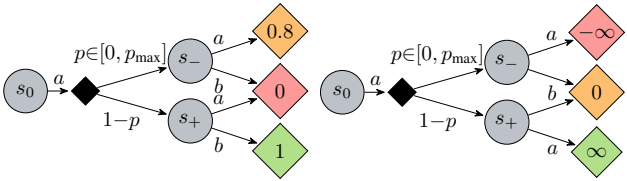

\paragraph{The worst-case transition function is measurement-dependent.}
\cref{eq:PR_uACNO} defines a worst-case transition function that depends on the complete action pair $\tilde{a} {=} \langle a,m \rangle$ rather than only on the control action $a$.
Thus, even though the uncertain transition function is independent of what measurement is chosen, the worst-case transition function is not.

As an example of why this dependency holds, consider the \uMVtwoenv \RACNO (\cref{fig:uMV}~left).
This environment has three states: an initial state $s_0$, and two next states $s_-$ and $s_+$ with different optimal actions $a$ and~$b$.
However, the reward for taking the optimal action in $s_-$ is lower than that in $s_+$.
We are interested in finding the worst $p$ when we measure in $s_0$ and when we do not.
When measuring, $s_-$ has a lower expected value than $s_+$, meaning the worst-case transition has $p {=} 1$.
When not measuring, however, this deterministic transition means the agent can safely pick action $a$ and receive a reward of $0.8$.
Instead, if $p$ is chosen closer to $0.5$, the expected return of taking action $a$ decreases, which gives worse expected returns overall.
Thus, the worst-case transition function depends on the chosen measuring action.

Intuitively, we find that for fully observable transitions (such as when measuring), the worst-case is simply given by maximizing worst-case outcomes.
However, for partially observable transitions (such as when not measuring), an \emph{unpredictable} outcome is often worse since this requires considering all possible outcomes for the next action.

\paragraph{High uncertainty can discourage measuring.}

The assumption that nature is adversarial 
(and thus chooses worst-case outcomes) %
is common for RMDPs.
However, in partially observable settings, nature does not only influence future predictions (via $P_\text{R}$) but (importantly) also predictions of past interactions (via $b_\text{R}$), and thus the current belief.
This may lead to overly conservative beliefs, especially if uncertainty is high.

As an example of overly conservative behavior induced by such beliefs, consider the \uMVenv  \RACNO~(\cref{fig:uMV} right).
As before, this environment has three states $s_0, s_{+}$, and $s_{-}$, where we interpret the latter two as a lucky and unlucky state.
In both, taking `safe' action $b$ leads to a neutral reward, while taking `risky' action $a$ gives an infinite positive reward in $s_+$ and an infinite negative reward in $s_-$.
We are interested in measuring behavior at different uncertainty intervals, as specified by $p_\text{max}$.
We notice the expected returns for $s_-$ are strictly lower than those of $s_+$, meaning an adversarial nature always chooses the highest possible probability $p$, regardless of whether the agent chooses to measure.
First, we assume $p_\text{max}{=}1$. 
This means the transition is deterministic, in which case measuring gives no additional information but still incurs a measuring cost and is thus sub-optimal.
Next, we assume $p_\text{max} {<} 1$. 
In this case, it is optimal to measure since this means spending a (finite) measuring cost to possibly achieve an infinite reward.
Counterintuitively, we find that high model uncertainty may lead to optimal policies taking fewer measurements than if model uncertainty is lower, even if measuring would alleviate this uncertainty. 
This property occurs even for finite returns and non-zero probabilities, as shown in \cref{sec:appendixProofs} \cite{appendix}.

The described behavior is the result of optimal robust policies (over-) optimizing for the worst case while not considering other possible outcomes.
This behavior makes sense in contexts where the environment must be considered adversarial, such as in security settings.
However, if policies are required to perform well on all possible models, such as when uncertainty represents confidence intervals, this over-optimization is unwanted behavior.
Moreover, if observations have additional value not captured by the model, such as to improve the model itself, we would want our policies to take measurements more leniently.
We will introduce a method to encourage such leniency in \Cref{sec:ControlRobustness}.
\section{Act-Then-Measure in \RACNOs}
\label{sec:uATM}

\begin{algorithm}[tb]
\caption{\textsc{Robust ATM Planner}}
\label{alg:ATM_Planner_Robust}
\begin{algorithmic}
\State Pre-compute $P_\text{RMDP}$ and $ Q_\text{RMDP}$
\State Initialise $b_0(s) = \delta(s, s_0)$
\While{\text{episode not completed}}
    \State Pick control action $a_t$ \Comment{\cref{eq:RobustControlAction}}
    \State Pick corresponding measuring action $m_t$ \Comment{\cref{eq:RobustMeasureCondition}}
    \State Execute $\tilde{a_t} = \langle a_t, m_t \rangle$
    \State Receive reward $\tilde{r}_t$ and observation $o_t$
    \State Determine next worst-case belief state $b_{t+1}$
    \Comment{\cref{eq:robustBeliefUpdate}}
\EndWhile
\State \textbf{return } $\sum_t \gamma^t \tilde{r}_t$
\end{algorithmic}
\end{algorithm}

In this section, we describe a method for finding approximate solutions for \RACNOs in a computationally tractable manner.
In particular, we extend the ATM heuristic \citep{DBLP:conf/aips/KraleS023} to an uncertain setting.

\begin{leftvrule}
    \textbf{The robust act-then-measure (RATM) heuristic:}
    \begin{enumerate*}
        \item chooses control-actions assuming that all (state) uncertainty will be resolved \emph{in the next} states (after one time-step); and
        \item chooses measuring-actions and updates beliefs assuming that all (state) uncertainty will be resolved \emph{after the next} states (after two time-steps).
    \end{enumerate*}
\end{leftvrule}
Since measurement actions only affect future state uncertainty, the first point of the heuristic allows us to \emph{pick control actions independently from measuring actions}.
Thus, our high-level strategy is given by \cref{alg:ATM_Planner_Robust}.
The remainder of this section will explain in detail how to perform each step in this algorithm.

\subsection{Choosing Control Actions}
Similar to the $Q_\text{MDP}$ heuristic \citep{DBLP:conf/icml/LittmanCK95}, the RATM heuristic means that returns of control actions can be approximated by those of the underlying RMDP:

\begin{equation}
\label{eq:RobustControlAction}
    Q_{\uATM}(b,a) = \sum_{s \in S} b(s) Q_{\text{RMDP}}(s,a) \approx Q_\text{R}(b,a) , %
\end{equation}
where $Q_{\uATM}$ denotes the approximate expected value when following the RATM heuristic, and $Q_{\text{RMDP}}$ and $Q_\text{R}$ denote optimal expected values for the \RACNO and its underlying RMDP, respectively.
Generally, $Q_{\text{RMDP}}$ can be efficiently pre-computed, allowing for faster policy computations than methods that fully consider partial observability.

\subsection{Computing Measuring Value}
Next, we need a method to pick measurement actions.
With noiseless and complete measurements, we define the \emph{robust measuring value} $\mvr$ as the difference in expected value between measuring and non-measuring actions:
\begin{equation}
\label{eq:mvr_exact}
    \mvr (b,a) = Q_{\uATM}(b, \langle a, 1 \rangle) - Q_{\uATM}(b, \langle a, 0 \rangle).
\end{equation}
Measuring is optimal for positive measuring values only, which yields the following measuring condition:
\begin{equation}
\label{eq:RobustMeasureCondition}
    m_{\text{RATM}}(b,a) = 
    \begin{cases}
        1 & \text{if } \mvr(b,a) \geq 0; \\
        0 & \text{otherwise},
    \end{cases}
\end{equation}
To compute $\mvr$, we need expressions for $Q_\uATM(b,\langle a, 1 \rangle)$ and $Q_\uATM(b,\langle a, 0 \rangle)$.
We note that the RATM heuristic means that for both, \cref{eq:RobustControlAction} can be applied to all beliefs \emph{after} the next one.
Thus, the Q-value when measuring is given as:
\begin{equation}
\label{eq:QMeasureRobust}
\begin{aligned}
        Q_\uATM & (b,\langle a, 1\rangle) = R(b,a) - c + \gamma \sum_s b(s) \\
    & \Big[ \sum_{s'}  P_\text{R}(s'|s,\langle a, 1 \rangle,b)  \max_{a}  Q_\text{RMDP}(s',a) \Big],
\end{aligned}
\end{equation}
with $R(b,a) {=} \sum_{s} b(s) R(s,a)$.
Here, we decide the next actions for each state separately, which we can only do if we take a measurement.
When not measuring, we must instead pick an optimal action considering all possible next states:
\begin{equation}
\label{eq:QNonMeasureRobust}
\begin{aligned}
     Q_\uATM & (b,\langle a, 0\rangle) = R(b,a) + \gamma \sum_s b(s)  \\
     & \Big[ \max_{a} \sum_{s'} P_\text{R}(s'|s,\langle a, 0 \rangle,b)  Q_\text{RMDP}(s',a) \Big].   
\end{aligned}
\end{equation}
Combining both equations, we write the robust measuring value of \cref{eq:mvr_exact} as follows:
\begin{equation}
\begin{aligned}
    \text{M}&\text{V}_\uATM (b,a) = -c + \max_{a' \in A} \gamma \sum_{s \in S} b(s) \\
    & \Big[ Q_\text{RMDP}(s,a) {-} \sum_{s' \in S} P_\text{R}(s'|s,\langle a, 0 \rangle, b) Q_\text{RMDP}(s',a') \Big].
\end{aligned}
\end{equation}
We notice that this equation contains only belief-independent and thus pre-computable quantities, with the exception of the transition function $P_\text{R}$.
This function is equal to $P_\text{RMDP}$ when measuring and otherwise given as:
\begin{equation}
\label{eq:PRobust}
\begin{aligned}
    P_\text{R} (s'|s,\langle a, 0 \rangle,b) = & \max_{a_b \in A} \min_{P(\cdot|s,a)\in \mathcal{P}(s,a,\cdot)} \sum_s b(s) \\
    & \Big[ \sum_{s' \in S} P(s'|s,a) Q_\text{RMDP}(s', a_b) \bigg].
\end{aligned}
\end{equation}
This equation can be tractably solved by a (non-convex) mixed integer program (MIP).
However, since this problem needs to be solved at every step, we find these computations take up the majority of the (online) runtime in our experiments.
With all quantities defined, we have fully described all steps in \cref{alg:ATM_Planner_Robust}.
Alternatively, we may define this algorithm as a policy:
\begin{equation}
\label{eq:RobustATMPolicy}
    \pi_\uATM(b) = \langle \max_{a \in A} Q_\text{R}\big(b,a\big) , m_\uATM\big(b, \max_{a \in A} Q_\text{R}(b,a)\big) \rangle.
\end{equation}

\section{Measurement Leniency}
\label{sec:ControlRobustness}

As outlined in \Cref{sec:ram-properties}, policies can exhibit (overly) conservative measuring behavior in \RAMs, particularly if model uncertainty is high.
In this section, we propose \emph{measurement leniency} to counteract this behavior.
Intuitively, measurement leniency means that agents choose control actions to optimize for the worst case, but may take extra measurements according to less pessimistic metrics, such as average expected returns.
Since the cost of extra measurements is bounded and predictable, measurement leniency might give sufficient robustness guarantees for many real-life applications while allowing less conservative behavior.
We formally define measurement leniency as follows:

\begin{definition}
Let $\pi$ be any policy. A corresponding \emph{measurement lenient} policy is any policy $\pi_\CR$ such that:
\begin{enumerate}[leftmargin=*]
    \item $\forall b, \pi(b) = \langle a, 1 \rangle \implies  \pi_\CR(b) = \langle a, 1 \rangle$, and
    \item $\forall b, \pi(b) = \langle a, 0 \rangle \implies  \pi_\CR(b) = \langle a, m \rangle$.
\end{enumerate}
\end{definition}
Using this definition, we define an optimal measurement lenient policy as one that maximizes the expected discounted scalarized return in some (less conservative) environment $\mathcal{M}_\CR$.
For example, if a probability distribution over transition functions is known, $\mathcal{M}_\CR$ could represent the most likely outcomes.
We formalize this as follows:

\begin{definition}
Let $\pi$ be a policy for an \RACNO $\mathcal{M}$, and $\Pi_\CR$ the corresponding set of measurement lenient policies.
Further, let $\mathcal{M}_\CR$ be any active-measure model with the same state- and action space as $\mathcal{M}$.
The \emph{optimal measurement lenient} policy $\pi_\CR^{*}$ with respect to $\mathcal{M}_\CR$ is given as:
\begin{equation}
    \pi_\CR^{*} = \arg \max_{\pi_\CR \in \Pi_\CR} \E_{\pi_\CR,\mathcal{M}_\CR} \sum_t \gamma^t \tilde{r}_t
\end{equation}
\end{definition}

\subsection{Computing Measurement Lenient Policies}

\begin{figure}[tb]
    \centering
    \begin{minipage}{.65\columnwidth}
        \resizebox{1\linewidth}{!}{\begin{tikzpicture}[
    shorten >=1pt,
    node distance=0.5cm and 2.5cm,
    on grid,
    >={Stealth[round]},
    every state/.style={fill=InkBlue}
    ]

    \def\distR{3.5cm}
    \node[tightS, fill=InkGrey]  (s_0)                                         {$s_0$};
    \node[tightP]    (p0)  [right= 1cm of s_0] {};
    
    \node[tightS, fill=InkGrey]          (s_+) [above right = of s_0]                      {$s_{+}$};
    \node[tightP]   (p+)  [right=1cm of s_+] {};
    
    \node[tightS, fill=InkGrey]          (s_-) [below right = of s_0]                      {$s_{-}$};
    
    \node[tightR, fill=InkGreen ]        (R+) [above right = 0.5cm and \distR of s_+]  {$1$};
    \node[tightR, fill=InkRed]           (R-)  [below right = 0.5cm and \distR of s_+]  {${-}1$};
    \node[tightR, fill=InkYellow ]       (R0) [below right = 0.5cm and \distR of s_-]  {$0$};

    \node[coordinate]                         (H0) [below=0cm of s_-] { };
    \node[coordinate]                         (H2) [right=0.75cm of R0] { };
    \node[coordinate]                         (H1) [below=0.6cm of R0] {};
    \node[coordinate]                         (H3) [right=0cm of R-] { };

    \draw [->] (s_-)  to[out=-45, in=180] node [above, xshift=-1.5cm, yshift=0.2cm] {a}  (H1) to[in=-90,out=0] (H2) to[in=0,out=90, ->] (R-) ;
  \path[->] (s_0)   edge node [above] {$a$}      (p0)
            (p0)    edge node [above]  {$0.5$} (s_+)
            (p0)    edge node [below]  {$0.5$} (s_-)
            (s_+)   edge node [above] {$a$} (p+)
            (p+)    edge node [above, yshift=0.1cm] {$p_{+} {\in}[0,1]$} (R+)
            (p+)    edge node [below] {$1-p_{+}$} (R-)
            (s_+)   edge  (R0)
            (s_-)   edge node [above,pos=0.15, yshift=0.1cm] {$b$} (R0)

            ;
\end{tikzpicture}}
    \end{minipage}
    \caption{
    A \RACNO where the measuring value of $\mathcal{M}_\CR$ for a measurement lenient policy is sub-optimal.
    Assuming an adversarial nature, the optimal control action in the states $s_+$ and $s_-$ is~$b$, meaning measuring in $s_0$ is sub-optimal.
    However, choosing a $\mathcal{M}_\CR$ with $p_+ {>} 0.5$ would yield a positive measuring value.
    }
    \label{fig:MVcrFail}
\end{figure}
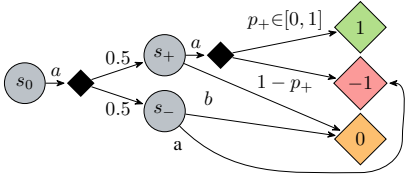

By construction, control actions of measurement lenient policies are equal to those in their base policies.
For any belief $b$, we may thus assume a control action $a_\text{R}(b)$ is given.

In order to make optimal measuring choices, we need to keep track of the current belief according to the dynamics of both $\mathcal{M}$ and $\mathcal{M}_\CR$.
For the latter, we denote $b_\CR$ as the current belief and $b'_\CR(b_\CR,\tilde{a})$ as the belief after taking action $\tilde{a}$ in belief $b_\CR$.
However, as shown in \cref{fig:MVcrFail}, simply using this belief to compute generic measuring value for $\mathcal{M}_\CR$ does not yield the correct behavior since this does not take into account that future control actions will be based on $\mathcal{M}$ instead.
To account for this, we first define the Q-value function for following a measurement lenient policy in $\mathcal{M}_\CR$:
\begin{equation}
\label{eq:QCR_general}
\begin{aligned}
    & Q_{\CR} (b_{\CR}, b, \langle a, m \rangle) = R(b_{\CR}, a) - C(m) \\
    & {+} \gamma \max_{m' \in M}  Q_{\CR} \Big( b'_{\CR}(b_\CR,\tilde{a}), b'_\text{R}(b,\tilde{a}), \langle a_\text{R}(b'_\text{R}(b,a)), m'\rangle \Big) 
\end{aligned}
\end{equation}
For complete and noiseless measurements, we express the measuring value for measurement lenient policies as:
\begin{equation}
\label{eq:mv_cr}
\begin{aligned}
    \mv_\CR(b_\CR, b, a) {=}  Q_{\CR} (b_{\CR}, b, \langle a, 1 \rangle) {-}  Q_{\CR} (b_{\CR}, b, \langle a, 0 \rangle)
\end{aligned}
\end{equation}
We first note that after a measurement $b'_\text{R}$ and $b'_\CR$ are always equal to the observation that has been made.
Thus, the Q-value when measuring can be expressed as follows:
\begin{equation}
\begin{aligned}
     Q_\CR & (b_\CR, b, \langle a, 1 \rangle) = R(b_\CR, a) - c \\
     & + \gamma \sum_{s \in S} b'_\CR(s|b_\CR, \tilde{a}) \max_{m' \in M} Q_\CR(s, s, \langle a_\text{R}(s), m'\rangle )
\end{aligned}
\end{equation}
Unfortunately, there is no trivial way to simplify our expression for non-measuring actions without making further assumptions about the policy choosing control actions.
For the general case where the fully robust policy is unknown, we thus propose to approximate our Q-values using the robust act-then-measure heuristic, as defined in \cref{sec:uATM}.
We restate the relevant part as follows:
\begin{leftvrule}
    \textbf{Measurement lenient approximation}: Choose measurement lenient measuring actions assuming all (state) uncertainty will be resolved \emph{after} the next state.
\end{leftvrule}
Using this approximation, we can simplify our equations by replacing future Q-values with those of the fully observable variant of the environment.
We denote this Q-value as $Q_\CRMDP$ and rewrite \cref{eq:QCR_general} as follows:
\begin{equation}
\begin{aligned}
    Q_{\CR} & (b_{\CR}, b, \langle a, m \rangle) \approx R(b_{\CR}, a) - C(m) \\
    & + \gamma \sum_{s' \in S}  Q_{\CRMDP}\Big(s', a_\text{R} \big(b'_\text{R}(b, \langle a, m \rangle )\big) \Big)
\end{aligned}
\end{equation}
We note that this expression does not require solving an optimization problem, meaning it can be computed quickly.
With this, we approximate \cref{eq:mv_cr} as follows:
\begin{equation}
\begin{aligned}
    \mv_\CR & (b_{C}, b, a) \approx - C(m) + \gamma \max_{a_b \in A} \sum_{s \in S} b_\CR(s) \\
    & \Big[ \max_{a \in A} Q_{\CRMDP} (s', a ) - Q_{\CRMDP} (s', a_\text{R}(b'_\text{R}(b,a))) \Big]
\end{aligned}
\end{equation}
For measurement lenient policies, the measuring condition requires the measuring value of both $\mathcal{M}_\CR$ and $\mathcal{M}$ to be non-negative, which gives:
\begin{equation}
    m_\CR(b_\CR, b, a) = \begin{cases}
    1 & \text{if } \mv_\CR(b_\CR,a) \geq 0 \\
      & \text{or } \mv_\text{R}(b,a) \geq 0 \\
    0 & \text{otherwise.}
    \end{cases}
\end{equation}

\begin{figure}[tb]
    \centering
    \includegraphics[width=0.65\columnwidth]{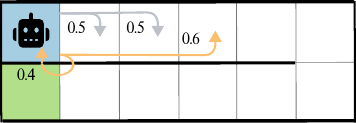}
    \caption{A $2{\times}5$ visualization of the \snakemaze environment.
    An agent traverses a snaking maze from the blue initial stat to the green goal state.
    It moves via \emph{safe} (grey) and \emph{risky} (yellow) actions with different stochastic effects.
    }
    \label{fig:SnakemazeEnv}
\end{figure}

\subsection{Regret of Measurement Lenient Policies}

One obvious downside of using measurement lenient policies is that their worst-case performance is generally lower.
However, we can show that their performance loss, as compared to their base policy, is bounded.
Intuitively, this bound follows from the fact that measurement lenient policies only take extra measurements, which decrease the total expected returns by (at most) $c$ per step.
We state this more formally:
\begin{theorem}
\label{thm:CRbound}
    Given an \RACNO $\mathcal{M}$ with complete and noiseless measurements and policy $\pi$.
    For any corresponding measurement lenient policy $\pi_\CR$, the following holds
    \begin{equation}
        \forall b: V(\pi,b) - V(\pi_\CR,b) \leq \sum_{n=0}^{\infty} \gamma^n c
    \end{equation}
\end{theorem}
We prove this theorem in \cref{sec:appendixProofs} \cite{appendix}.

\begin{figure*}[t]
\centering
        \includegraphics[width=.195\textwidth]{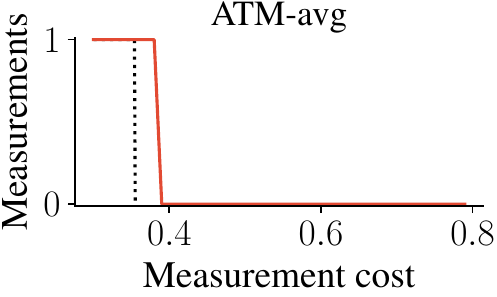}
        \includegraphics[width=.195\textwidth]{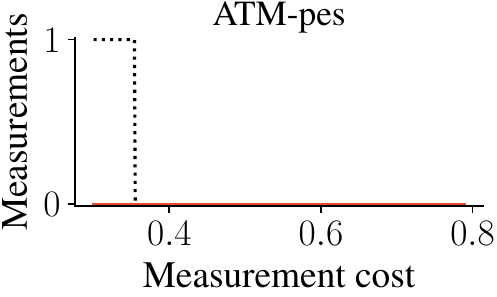}
        \includegraphics[width=.195\textwidth]{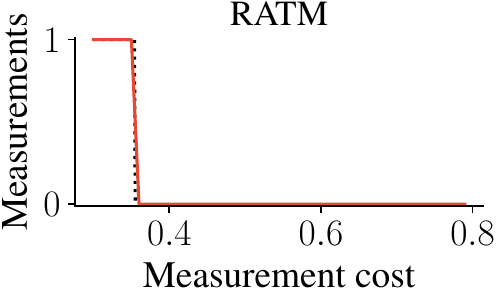}
        \includegraphics[width=.195\textwidth]{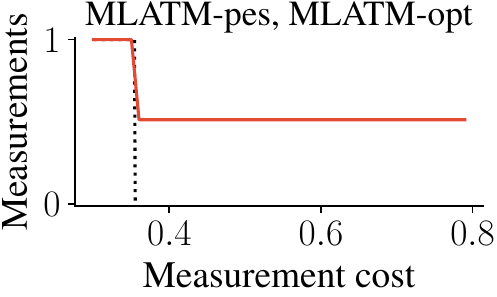}
        \includegraphics[width=.195\textwidth]{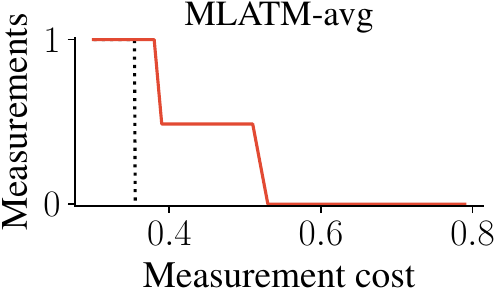}
    \caption{Mean number of measurements in \uMVtwoenv environment against measuring cost.
    Dotted lines show optimal behavior.}
    \label{fig:uMV2_measuring}
\end{figure*}

\begin{figure*}[tbh]
\centering
        \includegraphics[width=.195\textwidth]{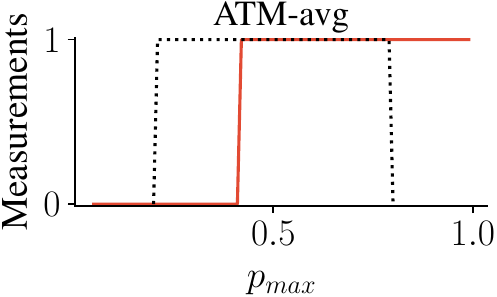}
        \includegraphics[width=.195\textwidth]{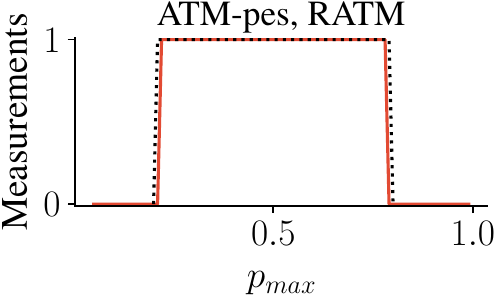}
        \includegraphics[width=.195\textwidth]{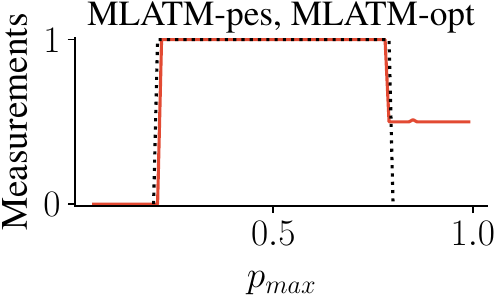}
        \includegraphics[width=.195\textwidth]{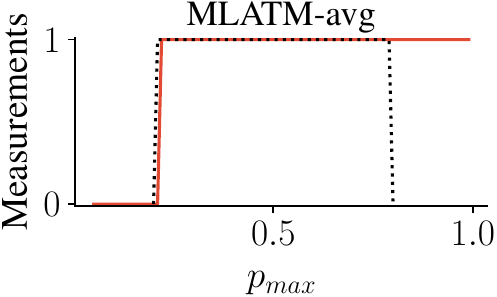}
    \caption{Mean number of measurements in \uMVenv environment against $p_\text{max}$.
    Dotted lines show optimal behavior.}
    \label{fig:uMV_measuring}
\end{figure*}

\section{Empirical Analysis}
\label{sec:experiments}

This section presents an empirical analysis of the behavior and performance of the proposed methods.
We run experiments on
\begin{enumerate*}
    \item the \uMVtwoenv and \uMVenv environments (both \cref{fig:uMV}); and
    \item two larger custom environments.
\end{enumerate*}
We test the following algorithms:
\begin{itemize}
    \item \uATM: the robust planning algorithm following the RATM heuristic, as described in \cref{alg:ATM_Planner_Robust}.
    \item \ucATM: the measurement lenient variant of \uATM. 
    We test three choices of $\mathcal{M}_\CR$: an \emph{optimistic}, \emph{pessimistic} and \emph{average} variant, denoted \ucATMopt, \ucATMpes and \ucATMavg and defined in \cref{sec:appendixResults} \cite{appendix}.
    \item \ATM: as a baseline, we use the \ATM planner from \citet{DBLP:conf/aips/KraleS023}. 
    We test two variants, denoted by \ATMavg and \ATMpes, which plan on the average-case environment and on the environment with transition function $P_\text{RMDP}$.
\end{itemize}
We provide code and data at \gitLink. 

\subsection{Behavior Evaluation}
We start with two small-scale experiments to determine how our algorithms
\begin{enumerate*}
    \item incorporate the effect of measuring in their worst-case transition function; and
    \item change measuring behavior for different sizes of uncertainty intervals.
\end{enumerate*}
For this, we run all algorithms on both the \uMVtwoenv and \uMVenv environments (\cref{fig:uMV}) and compare the algorithms with optimal behavior (depicted by dotted lines in the results), as calculated in \cref{sec:appendixResults} \cite{appendix}.

\cref{fig:uMV2_measuring} shows that \ATMavg and \ATMpes do not measure optimally in this environment, which shows they do not incorporate the effect of measuring on the transition function.
In contrast, \uATM measures optimally in this environment, while all measurement lenient variants take more measurements than optimal\footnotemark.
\cref{fig:uMV_measuring} shows that all algorithms except \ATMavg measure optimally when uncertainty is low, while the measurement lenient variants make (sub-optimal) measurements for high uncertainty, as expected.

\footnotetext{Surpisingly, some algorithms show non-deterministic measuring behavior. 
We explain this in \cref{sec:appendixResults} \cite{appendix}.}

\subsection{Performance Evaluation}

\begin{figure*}[tb]
        \hfill
        \includegraphics[width=.485\textwidth]{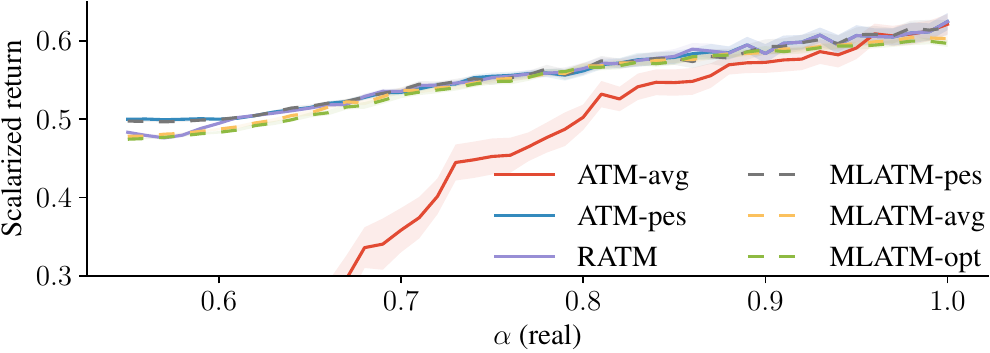}
        \hfill
        \includegraphics[width=.485\textwidth]{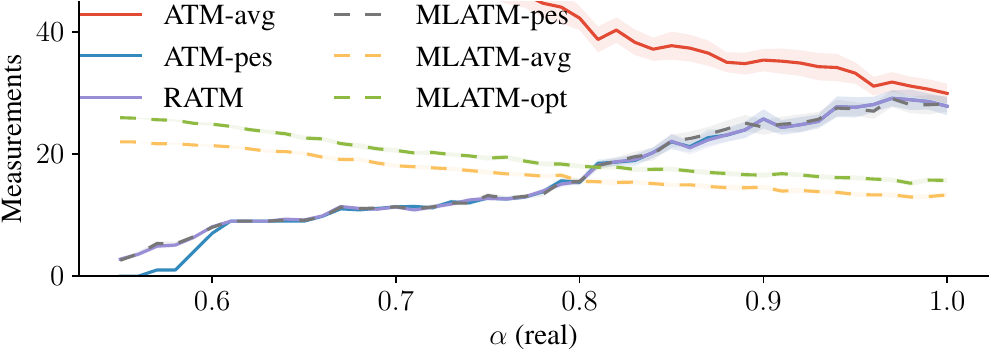}
        \hfill~
    \caption{%
    Returns and number of measurements in the \snakemaze environment, with $c{=}0.01$.
    Uncertainty is parameterized by confidence level $\alpha$ and decreases left-to-right.
    Lines show the mean of 50 runs, with $95\%$ confidence intervals shaded.}
    \label{fig:SnakemazeResults}
\end{figure*}

\begin{figure*}[tb]
        \hfill
        \includegraphics[width=.485\textwidth]{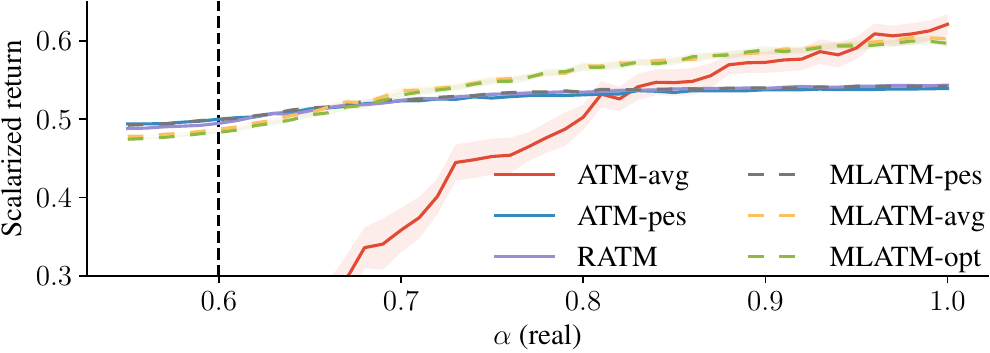}
        \hfill
        \includegraphics[width=.485\textwidth]{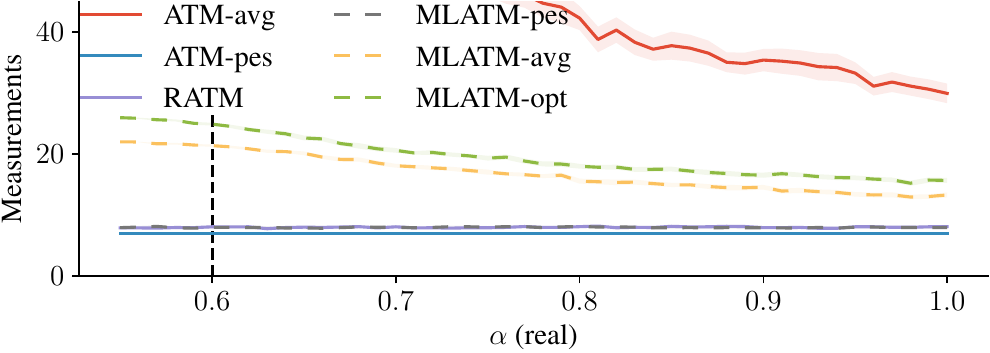}
        \hfill~
    \caption{%
    Returns and number of measurements in the \snakemaze environment, with $c =0.01$.
    Algorithms plan at uncertainty parametrized by $\alpha_p = 0.6$, but real dynamics is parametrized by $\alpha$.
    Thus, real uncertainty decreases left-to-right, while planning uncertainty is given by the dotted line.
    Lines show the mean of 50 runs, with $95\%$ confidence intervals shaded.}
    \label{fig:SnakemazeResultsMisspec}
\end{figure*}

Next, we test the performance of our algorithms on a custom environment called \snakemaze, which is designed to require conservative behavior.
A visualization is provided in \cref{fig:SnakemazeEnv}.
Starting in the top-left corner, the agent has to traverse a $10{\times}10$ snaking maze.
For each cardinal direction, the agent has the option to choose a \emph{safe} or \emph{risky} action.
A safe action has a $0.5$ chance to move the agent either 1 or two steps in the given direction, while the risky action has a $0.6$ chance of moving the agent three steps, but a $0.4$ chance of not moving the agent.
Thus, risky actions move the agent further on average, but even little uncertainty in the transition probabilities changes this.
The agent receives reward $1$ for reaching the goal and a small penalty for each prior step.

Following \citet{DBLP:conf/nips/Osogami12}, we parametrize uncertainty with a \emph{confidence level} $\alpha {\in} (0,1]$.
We define our uncertainty such that any transition probability is at most a factor $\nicefrac{1}{\alpha}$ larger than for some base transition function $P$, i.e., $ \forall s,s'{\in} S, A {\in} A \colon \mathcal{P}(s,a,s') {=} [0,\nicefrac{1}{\alpha}P(s'|s,a)]$.
Thus, $\alpha{=}1$ represents no model uncertainty, while uncertainty increases as $\alpha$ approaches 0.
Since computing the exact robust transition function for a \RACNO is intractable, we instead test our algorithms on the \emph{worst-case transition function assuming full observability}, i.e., using the transition function of the underlying RMDP.
This means that measurement-dependent worst-case transitions (as in the \uMVtwoenv environment) never occur, and we expect \ATMpes (which optimizes for the RMDP environment) to outperform the other algorithms.

\cref{fig:SnakemazeResults} (left) shows the scalarized returns of the algorithms at different confidence levels $\alpha$.
We see that \ATMavg gets outperformed by all other algorithms, while for $\alpha {<} 0.7$ \ucATMavg and \ucATMopt perform slightly worse than the more conservative algorithms.
This difference is caused by their different measuring behavior, as shown on the right.

Next, we are interested in how the algorithms perform if uncertainty is misspecified, i.e., if the algorithms plan with a different uncertainty than that of the real environment.
To test this, we let the algorithms plan on the \snakemaze environment with uncertainty parametrized by $\alpha_p {=} 0.6$, while we deploy the algorithms on the environment with $\alpha {\in} [0.55, 1]$.
\cref{fig:SnakemazeResultsMisspec} (left) shows the scalarized returns of the different algorithms.
We now see the advantage of measurement leniency: although \ucATMavg and \ucATMopt still perform slightly worse for large uncertainty, they outperform the more conservative algorithms for $\alpha {>} 0.7$.
This can be explained by the algorithms taking more measurements (as shown on the right), which allows them to take advantage of the more favorable environment.

\subsection{Scalability Evaluation}

\begin{figure*}[tb]
        \hfill
        \includegraphics[width=.485\textwidth]{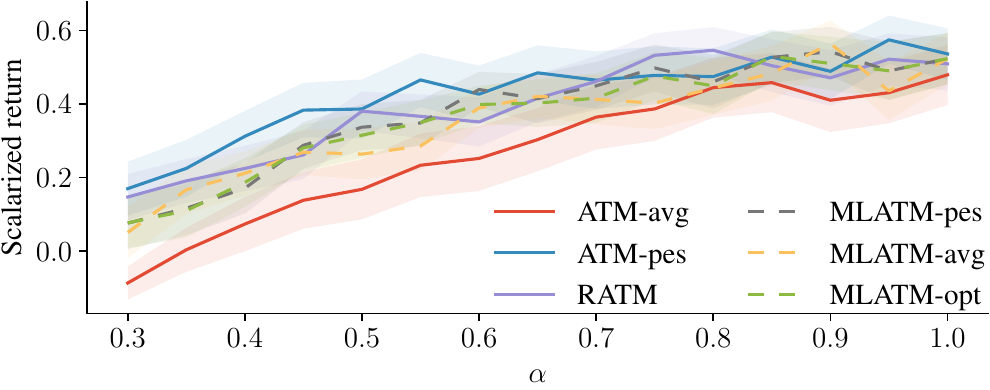}
        \hfill
        \includegraphics[width=.485\textwidth]{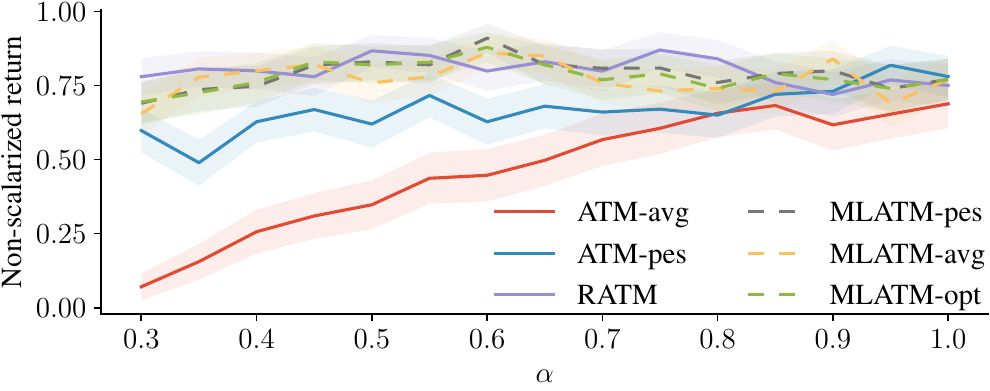}
        \hfill~
    \caption{%
    Scalarized and non-scalarized returns in the \drone environment, with $c{=}0.01$.
    Uncertainty is parameterized by confidence level $\alpha$ and decreases left-to-right.
    Lines show the mean of 100 runs, with $95\%$ confidence intervals shaded.}
    \label{fig:DroneResults}
\end{figure*}

\begin{figure*}[tb]
        \hfill
        \includegraphics[width=.485\textwidth]{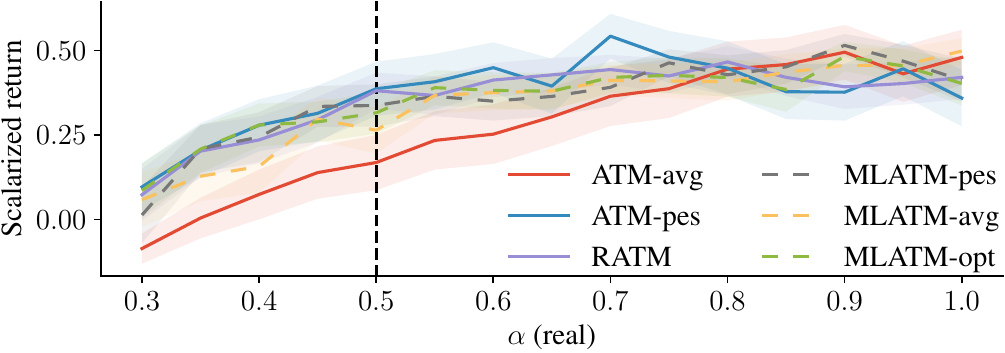}
        \hfill
        \includegraphics[width=.485\textwidth]{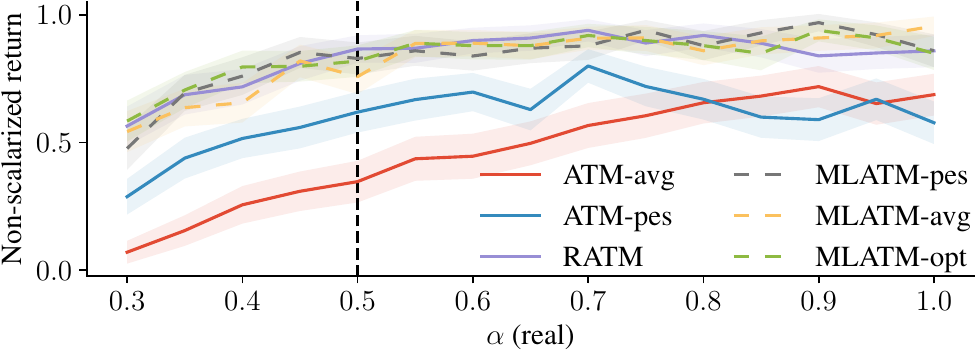}
        \hfill~
    \caption{%
    Scalarized and non-scalarized returns in the \drone environment, with $c =0.01$.
    Algorithms plan at uncertainty parametrized by $\alpha_p = 0.5$, but real dynamics is parametrized by $\alpha$.
    Thus, real uncertainty decreases left-to-right, while planning uncertainty is given by the dotted line.
    Lines show the mean of 50 runs, with $95\%$ confidence intervals shaded.}
    \label{fig:DroneResultsMisspec}
\end{figure*}

Lastly, to show our algorithms scale to larger environments, we run all algorithms on a custom \drone environment inspired by the example of \cref{fig:DroneExample}.
A full explanation of this environment is given in \cref{sec:appendixResults}.
The environment is a simplified and discretized motion model on a 2D grid, with $|S|{=}39,204$ states, $|A|{=}25$ actions, and up to $25$ successor states per state-action pair.
The agents receive a positive for reaching a certain set of goal states, and a small penalty for each prior step.
Like before, we parameterize uncertainty using confidence levels, and approximate the worst-case transition function by that of the underlying RMDP.

\cref{fig:DroneResults} shows both the scalarized (left) and non-scalarizd (right) returns of the algorithms at different confidence levels $\alpha$.
As for the \snakemaze environment, we find \ATMavg performs worse then the other algorithms, while \ATMpes slightly outperforms all (control) robust algorithms in terms of scalarized returns.
However, our algorithms outperform both \ATMavg and \ATMpes in terms of non-scalarized returns (i.e. returns excluding measuring costs).
Thus, our algorithms are more often able to reach the goal states, but take more measurements to do so.

For misspecified uncertainty, as shown in \cref{fig:DroneResultsMisspec}, we find similar results.
All algorithms except \ATMavg perform about on par in terms of scalarized return, while our algorithms outperform the baselines in terms of non-scalarized return.
Notably, we do not find a significant difference between robust and measurement lenient algorithms for this particular environment.
We suspect this is caused by a combination of
\begin{enumerate*}
    \item the transition function not being measurement-dependent; and
    \item the worst-case outcomes already incentivizing more frequent measuring (in contrast to the \snakemaze environment).
\end{enumerate*}
However, further analysis may be an interesting line of future research.

\subsection{Discussion}
Finally, we provide a summary of our findings.
\paragraph{R(C)ATM considers the effect of measuring.} 
\uATM performs optimally in the \uMVtwoenv environment, which is only possible when considering the effect of measuring. 
The control-robust variants only perform sub-optimally by taking more measurements, as expected.

\paragraph{R(C)ATM outscales previous methods.}
\uATM remains tractable and performs relatively well on the \drone environment, which is not solvable by prior robust methods.

\paragraph{Measurement leniency can prevent conservative measuring.}
Our experiments in the \uMVenv and \snakemaze environment show measurement leniency incentivizes measuring in some uncertain settings.
However, our experiments in the \drone environment show this is not always the case.
Thus, it is an open question whether measurement leniency affects performance in realistic settings.
\section{Related Work}
\label{sec:RelatedWork}

Active-measure MDPs, with noiseless and complete measurements, were introduced independently by \citet{DBLP:conf/nips/NamFB21} and \citet{DBLP:conf/ai/BellingerC0T21}, who both focussed on RL applications.
\citet{DBLP:conf/aips/KraleS023} introduced the ATM heuristic, which finds a tradeoff between performance and scalability.
Similar frameworks include those of \citet{DBLP:journals/ai/Doshi-VelezPR12}, who consider a setting where measurements return the optimal action, and \citet{DBLP:conf/nips/MateKXPT20}, who consider active measuring in a multi-armed bandit setting.
Active measuring has also been considered in settings where measuring cost and rewards are not combined, but measuring costs are constrained \citep{DBLP:conf/cdc/GhasemiT19} or minimized \citep{DBLP:conf/atva/BulychevCDLRR12}.
Lastly, some prior work considers setting with only measuring actions, where gathering information is the only goal. \citep{DBLP:conf/miccai/BernardinoJLCSC22,DBLP:conf/ewrl/Araya-LopezBTC11}.

Although \RAMs have not been studied previously, the more general framework of RPOMDPs has.
\citet{@osogamiRobustPartiallyObservable2015} and \citet{DBLP:conf/aaai/Cubuktepe0JMST21} describe methods for solving adversarial RPOMDPs, using value iteration and finite state controllers, respectively.
However, the former method scales poorly to large environments, while the latter only produces policies with small memory.
Next, \citet{@rasouliRobustPartiallyObservable2018} gives an in-depth analysis of RPOMDPs with different assumptions, and \citet{ElineThesis} describes how to represent uncertain beliefs without assuming adversariality.
However, their methods are intractable for the sizes of environments considered here.

\section{Conclusion}
We introduced \RAMs as a framework to represent active measuring environments with model uncertainty.
To solve a specific subset of \RAMs, we re-defined the act-then-measure heuristic for generic active-measure environments for uncertain settings.
Next, we proposed \emph{measurement leniency} to deal with overly conservative measuring behavior.
We empirically evaluate both generic and measurement lenient variants of our algorithm, showing they are tractable and outperform non-robust baselines.

Future work will focus on making \uATM more scalable, for example, by finding a way to approximately solve \cref{eq:PRobust}, the current computational bottleneck.
Moreover, we will explore more general (robust) active-measure environments with partial or noisy measurements.

\section*{Acknowledgments}
This research has been partially funded by NWO grant NWA.1160.18.238 (PrimaVera) and the ERC Starting Grant 101077178 (DEUCE).

\bibliography{sources.bib}

\newcommand{\arXiv}[1]{\emph{arXiv preprint arXiv:#1}}
\begin{thebibliography}{27}
\providecommand{\natexlab}[1]{#1}

\bibitem[{Araya{-}L{\'{o}}pez et~al.(2011)Araya{-}L{\'{o}}pez, Buffet, Thomas,
  and Charpillet}]{DBLP:conf/ewrl/Araya-LopezBTC11}
Araya{-}L{\'{o}}pez, M.; Buffet, O.; Thomas, V.; and Charpillet, F. 2011.
\newblock Active Learning of {MDP} Models.
\newblock In \emph{{EWRL}}, volume 7188 of \emph{Lecture Notes in Computer
  Science}, 42--53. Springer.

\bibitem[{Badings et~al.(2023)Badings, Romao, Abate, Parker, Poonawala,
  Stoelinga, and Jansen}]{DBLP:journals/jair/BadingsRAPPSJ23}
Badings, T.~S.; Romao, L.; Abate, A.; Parker, D.; Poonawala, H.~A.; Stoelinga,
  M.; and Jansen, N. 2023.
\newblock Robust Control for Dynamical Systems with Non-Gaussian Noise via
  Formal Abstractions.
\newblock \emph{J. Artif. Intell. Res.}, 76: 341--391.

\bibitem[{Bellinger et~al.(2021)Bellinger, Coles, Crowley, and
  Tamblyn}]{DBLP:conf/ai/BellingerC0T21}
Bellinger, C.; Coles, R.; Crowley, M.; and Tamblyn, I. 2021.
\newblock Active Measure Reinforcement Learning for Observation Cost
  Minimization.
\newblock In \emph{Canadian Conference on {AI}}. Canadian Artificial
  Intelligence Association.

\bibitem[{Bernardino et~al.(2022)Bernardino, Jonsson, Loncaric, Castellote,
  Sitges, Clarysse, and Duchateau}]{DBLP:conf/miccai/BernardinoJLCSC22}
Bernardino, G.; Jonsson, A.; Loncaric, F.; Castellote, P.~M.; Sitges, M.;
  Clarysse, P.; and Duchateau, N. 2022.
\newblock Reinforcement Learning for Active Modality Selection During
  Diagnosis.
\newblock In \emph{{MICCAI} {(1)}}, volume 13431 of \emph{Lecture Notes in
  Computer Science}, 592--601. Springer.

\bibitem[{Bovy(2023)}]{ElineThesis}
Bovy, E. 2023.
\newblock The Underlying Belief Model of Uncertain Partially Observable
  {M}arkov Decision Processes.
\newblock \emph{Master Thesis}.

\bibitem[{Bulychev et~al.(2012)Bulychev, Cassez, David, Larsen, Raskin, and
  Reynier}]{DBLP:conf/atva/BulychevCDLRR12}
Bulychev, P.~E.; Cassez, F.; David, A.; Larsen, K.~G.; Raskin, J.; and Reynier,
  P. 2012.
\newblock Controllers with Minimal Observation Power (Application to Timed
  Systems).
\newblock In \emph{{ATVA}}, volume 7561 of \emph{Lecture Notes in Computer
  Science}, 223--237. Springer.

\bibitem[{Cubuktepe et~al.(2021)Cubuktepe, Jansen, Junges, Marandi, Suilen, and
  Topcu}]{DBLP:conf/aaai/Cubuktepe0JMST21}
Cubuktepe, M.; Jansen, N.; Junges, S.; Marandi, A.; Suilen, M.; and Topcu, U.
  2021.
\newblock Robust Finite-State Controllers for Uncertain {POMDP}s.
\newblock In \emph{{AAAI}}, 11792--11800. {AAAI} Press.

\bibitem[{Doshi{-}Velez, Pineau, and
  Roy(2012)}]{DBLP:journals/ai/Doshi-VelezPR12}
Doshi{-}Velez, F.; Pineau, J.; and Roy, N. 2012.
\newblock Reinforcement learning with limited reinforcement: Using Bayes risk
  for active learning in {POMDP}s.
\newblock \emph{Artif. Intell.}, 187: 115--132.

\bibitem[{Ghasemi and Topcu(2019)}]{DBLP:conf/cdc/GhasemiT19}
Ghasemi, M.; and Topcu, U. 2019.
\newblock Online Active Perception for Partially Observable {M}arkov Decision
  Processes with Limited Budget.
\newblock In \emph{{CDC}}, 6169--6174. {IEEE}.

\bibitem[{Jimenez-Roa et~al.(2022)Jimenez-Roa, Heskes, Tinga, Molegraaf, and
  Stoelinga}]{jimenez2022deterioration}
Jimenez-Roa, L.~A.; Heskes, T.; Tinga, T.; Molegraaf, H.~J.; and Stoelinga, M.
  2022.
\newblock Deterioration modeling of sewer pipes via discrete-time {M}arkov
  chains: A large-scale case study in the Netherlands.
\newblock In \emph{32nd European Safety and Reliability Conference, ESREL 2022:
  Understanding and Managing Risk and Reliability for a Sustainable Future},
  1299--1306.

\bibitem[{Kaelbling, Littman, and
  Cassandra(1998)}]{DBLP:journals/ai/KaelblingLC98}
Kaelbling, L.~P.; Littman, M.~L.; and Cassandra, A.~R. 1998.
\newblock Planning and Acting in Partially Observable Stochastic Domains.
\newblock \emph{Artif. Intell.}, 101(1-2): 99--134.

\bibitem[{Kormushev, Calinon, and
  Caldwell(2013)}]{DBLP:journals/robotics/KormushevCC13}
Kormushev, P.; Calinon, S.; and Caldwell, D.~G. 2013.
\newblock Reinforcement Learning in Robotics: Applications and Real-World
  Challenges.
\newblock \emph{Robotics}, 2(3): 122--148.

\bibitem[{Krale, Sim{\~{a}}o, and Jansen(2023)}]{DBLP:conf/aips/KraleS023}
Krale, M.; Sim{\~{a}}o, T.~D.; and Jansen, N. 2023.
\newblock Act-Then-Measure: Reinforcement Learning for Partially Observable
  Environments with Active Measuring.
\newblock In \emph{{ICAPS}}, 212--220. {AAAI} Press.

\bibitem[{Krale et~al.(2024)Krale, Simão, Tumova, and Jansen}]{appendix}
Krale, M.; Simão, T.~D.; Tumova, J.; and Jansen, N. 2024.
\newblock Robust Active Measuring under Model Uncertainty.
\newblock \arXiv{.....}

\bibitem[{Lei et~al.(2020)Lei, Tan, Zheng, Liu, Zhang, and
  Shen}]{DBLP:journals/comsur/LeiTZLZS20}
Lei, L.; Tan, Y.; Zheng, K.; Liu, S.; Zhang, K.; and Shen, X. 2020.
\newblock Deep Reinforcement Learning for Autonomous Internet of Things: Model,
  Applications and Challenges.
\newblock \emph{{IEEE} Commun. Surv. Tutorials}, 22(3): 1722--1760.

\bibitem[{Littman, Cassandra, and Kaelbling(1995)}]{DBLP:conf/icml/LittmanCK95}
Littman, M.~L.; Cassandra, A.~R.; and Kaelbling, L.~P. 1995.
\newblock Learning Policies for Partially Observable Environments: Scaling Up.
\newblock In \emph{{ICML}}, 362--370. Morgan Kaufmann.

\bibitem[{Mate et~al.(2020)Mate, Killian, Xu, Perrault, and
  Tambe}]{DBLP:conf/nips/MateKXPT20}
Mate, A.; Killian, J.~A.; Xu, H.; Perrault, A.; and Tambe, M. 2020.
\newblock Collapsing Bandits and Their Application to Public Health
  Intervention.
\newblock In \emph{NeurIPS}.

\bibitem[{Nam, Fleming, and Brunskill(2021)}]{DBLP:conf/nips/NamFB21}
Nam, H.~A.; Fleming, S.~L.; and Brunskill, E. 2021.
\newblock Reinforcement Learning with State Observation Costs in
  Action-Contingent Noiselessly Observable {M}arkov Decision Processes.
\newblock In \emph{NeurIPS}, 15650--15666.

\bibitem[{Nilim and Ghaoui(2005)}]{DBLP:journals/ior/NilimG05}
Nilim, A.; and Ghaoui, L.~E. 2005.
\newblock Robust Control of {M}arkov Decision Processes with Uncertain
  Transition Matrices.
\newblock \emph{Oper. Res.}, 53(5): 780--798.

\bibitem[{Osogami(2012)}]{DBLP:conf/nips/Osogami12}
Osogami, T. 2012.
\newblock Robustness and risk-sensitivity in {M}arkov decision processes.
\newblock In \emph{{NIPS}}, 233--241.

\bibitem[{Osogami(2015)}]{@osogamiRobustPartiallyObservable2015}
Osogami, T. 2015.
\newblock Robust partially observable {M}arkov decision process.
\newblock In Bach, F.; and Blei, D., eds., \emph{Proceedings of the 32nd
  International Conference on Machine Learning}, volume~37 of \emph{Proceedings
  of Machine Learning Research}, 106--115. Lille, France: PMLR.

\bibitem[{Puterman(1994)}]{DBLP:books/wi/Puterman94}
Puterman, M.~L. 1994.
\newblock \emph{{M}arkov Decision Processes: Discrete Stochastic Dynamic
  Programming}.
\newblock Wiley Series in Probability and Statistics. Wiley.

\bibitem[{Rasouli and Saghafian(2018)}]{@rasouliRobustPartiallyObservable2018}
Rasouli, M.; and Saghafian, S. 2018.
\newblock Robust Partially Observable {M}arkov Decision Processes.
\newblock \emph{SSRN Electronic Journal}.

\bibitem[{Suilen et~al.(2020)Suilen, Jansen, Cubuktepe, and
  Topcu}]{DBLP:conf/ijcai/Suilen0CT20}
Suilen, M.; Jansen, N.; Cubuktepe, M.; and Topcu, U. 2020.
\newblock Robust Policy Synthesis for Uncertain POMDPs via Convex Optimization.
\newblock In \emph{{IJCAI}}, 4113--4120. ijcai.org.

\bibitem[{Sunberg and Kochenderfer(2022)}]{DBLP:journals/tits/SunbergK22}
Sunberg, Z.; and Kochenderfer, M.~J. 2022.
\newblock Improving Automated Driving Through {POMDP} Planning With Human
  Internal States.
\newblock \emph{{IEEE} Trans. Intell. Transp. Syst.}, 23(11): 20073--20083.

\bibitem[{Wiesemann, Kuhn, and Rustem(2013)}]{DBLP:journals/mor/WiesemannKR13}
Wiesemann, W.; Kuhn, D.; and Rustem, B. 2013.
\newblock Robust {M}arkov Decision Processes.
\newblock \emph{Math. Oper. Res.}, 38(1): 153--183.

\bibitem[{Yu et~al.(2023)Yu, Liu, Nemati, and Yin}]{DBLP:journals/csur/YuLNY23}
Yu, C.; Liu, J.; Nemati, S.; and Yin, G. 2023.
\newblock Reinforcement Learning in Healthcare: {A} Survey.
\newblock \emph{{ACM} Comput. Surv.}, 55(2): 5:1--5:36.

\end{thebibliography}

\clearpage

\appendix

\section{Extended Empirical Analysis}
\label{sec:appendixResults}

In this section, we give a more in-depth explanation and analysis of both our experimental setup and results.
All experiments were run on a laptop with Intel i7-10875H 2.3 GHz processor and 32 GB of RAM.
All code is written in Python and available at \gitLink.

\subsection{Algorithms}

Our experimental analysis focuses on three algorithms: \uATM, \ucATM, and \ATM. 
For the latter two, we require a generic ACNO-MDP for planning, which we vary in our experiments.
In particular, we define a \emph{pessimistic, optimistic} and \emph{average} ACNO-MDP, as follows:

\begin{itemize}
    \item \emph{Pessimistic} (pes): the worst-case environment within $\mathcal{M}$ assuming full observability, i.e. $\mathcal{M}_\text{RMDP}$.    
    \item \emph{Average} (avg): the environment where each transition probability is the average of the highest and lowest valid probability in $\mathcal{M}$.
    Its transition function is defined as:
    \begin{equation}
    \begin{aligned}
        \forall s, s', a: & \mathcal{P}(s'|s,a) = [p_\text{min}, p_\text{max}] \\
        & \rightarrow P(s'| s,a) = (p_\text{min} + p_\text{max})/2 .
    \end{aligned}
    \end{equation}
    Note that this is not always a valid transition function but is for all our experiments, and is otherwise easy to normalize.    
    \item \emph{Optimistic} (opt): the best-case environment within $\mathcal{M}$ assuming full observability.
    We define it in a similar fashion as the robust transition function (\cref{eq:PR_uACNO,eq:VR_uACNO}) but we choose a transition function that maximizes returns instead of minimizing them:
    \begin{equation}
    \begin{aligned}
         & P_\text{opt} (s'|s,\tilde{a}) = \arg \sup_{P_\text{opt}(\cdot|s,\tilde{a}) \in \mathcal{P}(\cdot|s,a)} V_\text{opt}(s), \\
        & V_\text{opt}(s) = \max_{a \in A} R(s, a) + \gamma \sum_{s' \in S} P_\text{opt}(s'|s,a) V_\text{opt}(s').
    \end{aligned}
    \end{equation}
\end{itemize}

\subsection{Setup and Optimal Policies for Behavioral Experiments}

\begin{figure}
    \centering
    \begin{minipage}{.5\columnwidth}
        \resizebox{1\linewidth}{!}{\begin{tikzpicture}[
    shorten >=1pt,
    node distance=0.5cm and 2.75cm,
    on grid,
    >={Stealth[round]},
    every state/.style={fill=InkBlue}
    ]

    \def\distR{1.5cm}
    \node[tightS, fill=InkGrey]  (s_0)                                         {$s_0$};
    \node[tightP]    (p0)  [right= 1cm of s_0] {};
    \node[tightS, fill=InkGrey]          (s_-) [above right = of s_0]                      {$s_{-}$};
    \node[tightS, fill=InkGrey]          (s_+) [below right = of s_0]                      {$s_{+}$};
    
    \node[tightR, shape=diamond, fill=InkRed, inner sep=0 ]        (R-) [above right = 0.5cm and \distR of s_-]  {${-}1$};
    \node[tightR, fill=InkYellow]           (R5)  [below right = 0.5cm and \distR of s_-]  {$0$};
    \node[tightR, shape=diamond, fill=InkGreen ]        (R+) [below right = 0.5cm and \distR of s_+]  {$1$};

  \path[->] (s_0)   edge node [above] {a}      (p0)
            (p0)    edge node [above, yshift=0.1cm, xshift=-0.2cm]  {$p {\in} [0,p_{\text{max}}]$} (s_-)
            (p0)    edge node [below, xshift=-0.2cm]  {$1{-}p$} (s_+)
            (s_-)   edge node [above] {a} (R-)
            (s_-)   edge node [below] { } (R5)
            (s_+)   edge node [above] {b} (R5)
            (s_+)   edge node [below] {a} (R+)
            ;
\end{tikzpicture}}
    \end{minipage}
    \caption{A variant of the \uMVenv environment, as used in our experimental analysis.}
    \label{fig:uMVEnv_param}
\end{figure}
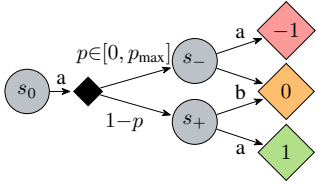

\begin{figure}
    \centering
    \includegraphics[width=0.8\columnwidth]{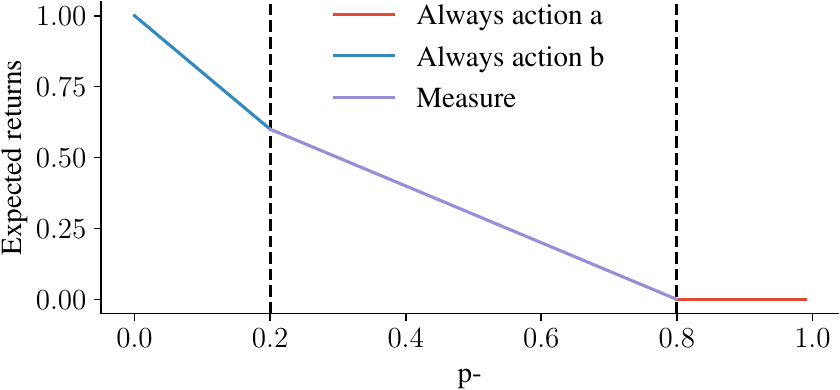}
    \caption{Expected returns for the optimal policy of the \uMVenv environment against $p_-$, for measuring cost $s=0.2$. Colours represent different strategies.}
    \label{fig:uMVline}
\end{figure}

To start, let us describe in more detail our experimental setup for our testing on both the \uMVtwoenv and \uMVenv environments.

For the \uMVtwoenv environment, we test the algorithms for different measuring costs $c$.
To find the optimal policy for this environment, we first notice that in $s_0$ only one action is permitted, and measuring in the second step is never optimal.
Thus, we only need to analyze the returns for measuring and not measuring in $s_0$ for each possible action in the second step.
When measuring, we notice that $s_+$ has a higher expected return than $s_-$.
Thus, in this case, the worst-case transition function deterministically brings us to $s_-$, meaning $Q(s_0, \langle a, 0 \rangle) = 0.8$.
When not measuring, the worst-case transition function should be such that the expected value for actions $a$ and $b$ is equal.
These values are given by $0.8p$ and $(1-p)$ respectively, meaning the worst-case transition probability is given by $p=1/(1+0.8)$ and the expected return is given by $Q(s_0, \langle a, 1 \rangle) =0.8/(1+0.8)$.
Combining the two expected returns, we find measuring value is given by $\mvr = 0.8 (1- \frac{1}{1+0.8}) - c = 0.3\bar{5} -c$.

Next, for the \uMVenv environment, we firstly make a slight alteration by setting the worst- and best-case outcomes to $\pm 1$ instead of $\pm \infty$ (see \cref{fig:uMVEnv_param}).
We choose measuring cost $c=0.2$, and vary $p_\text{max}$.
As mentioned earlier, the \uMVenv environment has the interesting property that measuring becomes sub-optimal for large enough $p_\text{max}$.
To determine for which interval measuring is optimal, we first note (like before) that measuring in the second step is never optimal.
Depending on $p_{\text{max}}$, then, there are three optimal strategies for this environment.
If $p_{\text{max}}$ is sufficiently large or small, the state uncertainty is small and measuring in the second step is not worth the cost.
In these cases, we get strategies that never measure and take only action $a$ or $b$, which lead to expected returns of $-p + (1-p) = 1 -2p$ and $0$, respectively.
Another strategy is to always measure the second step and decide the next action based on our observation: take action $a$ in $s_+$, and $b$ in $s_-$
This yields an expected return of $(1-p) -c$.
An optimal policy, then, simply chooses the strategy with the highest return according to $p_{\text{max}}$ and $c$.
For $c=0.2$, \cref{fig:uMVline} shows the returns of this policy for different probabilities $p_{\text{max}}$.

\subsection{Stochastic Measuring by Measurement Lenient Policies.}

In the \uMVtwoenv and \uMVenv environments, we see the control-robust policies do not follow optimal measuring behavior but instead keep making measurements even after this becomes suboptimal, which is expected.
Surprisingly, however, the algorithms show \emph{inconsistent} measuring behavior, i.e. measure only for $50\%$ of episodes, at certain measuring costs.
To explain why this happens, we focus on the \uMVtwoenv environment.
We first notice that when not measuring in the first step, the expected returns for taking action $a$ or $b$ in the second step are identical.
In such cases, we have implemented our algorithm to randomly pick one of the two actions to calculate $\mvr$.
However, in $\mathcal{M}_\CR$, these actions may not have the same value, which might give a higher measuring value for one action than the other.
In the pessimistic variant, for example, $p{=}1$, meaning action \ucATMpes will measure only if action $b$ is chosen.
This causes sporadic measuring, which explains the behavior seen in our results.
Similar behavior occurs with the two other measurement lenient variants and on the \uMVenv environment.
If undesirable, this behavior could be prevented by choosing control actions based on $\mathcal{M}_\CR$ in case of ties.

\subsection{Drone Environment}

\begin{figure}[tb]
    \centering
    \includegraphics[width=0.75\columnwidth]{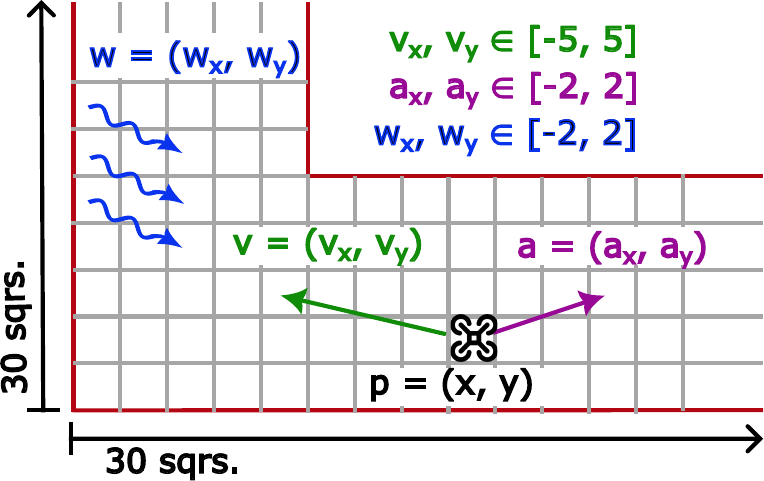}
    \caption{Visualisation of the \drone environment. 
    A drone plans a path through a 2D grid world, where states consist of both its current position and velocity.
    Uncertainty is added via stochastic (wind) disturbances.}
    \label{fig:DroneEnv}
\end{figure}

In this section, we describe in detail the \drone environment used in our experimental analysis: the results themselves are given in the next section.
A visualization of the environment is shown in \cref{fig:DroneEnv}.
We note that the abstraction is meant as an easy-to-understand representation rather than a robust abstraction of a real problem.
For more formal methods, see e.g. \citet{DBLP:journals/jair/BadingsRAPPSJ23}.

\subsubsection{Discretized Dynamics}
To start, let us define a simple generic framework to represent continuous movement using discrete variables.
Firstly, we represent the drone's position as coordinates in a grid, meaning the drone's position can be expressed as $p=\langle  x,y\rangle$, with $x,y \in \mathbb{Z}$ grid coordinates.
For simplicity, we'll assume independence in dynamics between these two directions.
With this assumption, we define velocity as $v=\langle v_x,v_y\rangle \in \mathbb{Z}^2$, which represents the number of grid cells moved each timestep.
Lastly, we define acceleration $a=\langle a_x,a_y\rangle \in \mathbb{Z}^2$ and perturbations $w=\langle w_x,w_y \rangle \in \mathbb{Z}^2$, which influence the change in velocity for each time-step.
The (discrete-time) dynamics of our system are described as follows:

\begin{equation}
\label{eq:DroneDynamics}
\begin{aligned}
    x_t &{=} x_{t-1} + \lfloor \frac{v_{x,t-1} + v_{x,t}}{2} \rfloor, &
    v_{x,t} &{=} v_{x,t-1} + a_{x,t} + w_{x,t}, \\
    y_t &{=} y_{t-1} + \lfloor \frac{v_{y,t-1} + v_{y,t}}{2} \rfloor, &
    v_{y,t} &{=} v_{y,t-1} + a_{y,t} + w_{y,t}.
\end{aligned}   
\end{equation}

For the positions variables $i\in \{x,y\}$, we note that $\lfloor \frac{v_{i,t-1} + v_{i,t}}{2} \rfloor$ represents the average velocity during timestep $t$, which could be a fraction and thus needs to be rounded.

This simple framework can be used to express an MDP with states of form $s = \langle p, v \rangle = \langle x,y,v_x,v_y \rangle$, actions of form $a = \langle a_x, a_y \rangle$, any reward function $R$, and a transition function $P_{W}$ following \cref{eq:DroneDynamics} for a given probability distribution $W:\Delta (\mathbb{Z})$ over perturbations $w$.

\subsubsection{A Drone in a Corridor}

To make these general dynamics concrete for our environment, we first represent our corridor by two overlapping rectangles of $6 \times 30$ squares \footnotemark, with all values outside this area expressed as one sink-state $s_\text{sink}$.
Next, we restrict velocities to be within the range $[-v_\text{max}, v_\text{max}] = [-5,5]$, with any value outside this range set to the closest valid value.
In a similar fashion, we restrict accelerations to be within the range $[-a_\text{max}, a_\text{max}] = [-2,2]$, giving us a finite action space.
We choose a probability distribution $W$ based loosely on a Gaussian:

\footnotetext{We can imagine these grid cells have length $l\approx 20cm$, in which case we have a $1.2m$ wide corridor. 
Further assuming times-steps of $1s$, we get speeds which increment by $0.2m/s$ up to a maximum of $1m/s$, and accelerations up to a maximum of $0.4 m/s^2$.}

\begin{equation}
    w_{i \in \{x,y\}, t} = 
    \begin{cases}
        0 & \text{ with probability } 0.68, \\
        \pm 1 & \text{ with probabilities } 0.14, \\
        \pm 2 & \text{ with probabilities } 0.02.
    \end{cases}
\end{equation}

We define an initial state $s_0 = \langle x{=}29, y{=}2, v_x{=}v_y{=}0 \rangle$ corresponding to a motionless drone at one end of the corridor.
Lastly, we define a simple reward function that yields 1 only if a goal area at the other end of the corridor is reached:

\begin{equation}
    R(s,a) = 
    \begin{cases}
        1 & \text{ if } y > 27 \text{ and } s \neq s_\text{sink},\\
        0 & \text{otherwise.}
    \end{cases}
\end{equation}

Episodes end either if $s = s_\text{sink}$, representing a crash, or the goal area is reached.
The full MDP has a state space $S$ of size $|S| = 39.204$ and action space $A$ of size $|A| = 25$, with each state-action pair having up to 25 successor states.

\subsubsection{From MDP to \RACNO}
To turn our environment into an \RACNO, we start by adding model uncertainty.
Taking inspiration from \citet{DBLP:conf/nips/Osogami12}, we take the MDP defined above and add \emph{confidence intervals} parametrized by a \emph{confidence level} $\alpha \in [0,1]$.
Using this, we define our uncertain transition function as follows:

\begin{equation}
\label{eq:confMDP_uP}
    \mathcal{P}(s'|s,a) = [0, p] \text{, with } p = \min \{\frac{1}{\alpha} P(s'|s,a), 1 \}
\end{equation}

Robust MDPs with such transition functions can be used to deal with risk averseness, which is explained more fully in \citet{DBLP:conf/nips/Osogami12}.
Using this framework, we find both independent directions follow a conditional probability distribution given as:

\begin{equation}
\begin{aligned}
    &\forall i\in \{x,y\}: \text{Pr}(i_t | , i_{t-1}, v_{i,t-1}, a_{i,t-1}) = \\
    &\begin{cases}
        [0, \max (\frac{0.02}{\alpha},1)] & \text{ if } i_t = i_{t-1} + \lfloor v_{i,t-1} + a_{i,t-1} \rfloor \\
        [0, \max (\frac{0.14}{\alpha},1)] & \text{ if } i_t = i_{t-1} + \lfloor v_{i,t-1} + a_{i,t-1} \pm 1 \rfloor \\
        [0, \max (\frac{0.68}{\alpha},1)] & \text{ if } i_t = i_{t-1} + \lfloor v_{i,t-1} + a_{i,t-1} \pm 2 \rfloor
    \end{cases}
\end{aligned}
\end{equation}

To turn this into a \RACNO, we simply add a measuring cost $c$, and all components are defined.

\subsection{Drone Environment Results}

\section{Extended proofs}
\label{sec:appendixProofs}

\subsection{Uncertainty discouraging measuring}

Here, we show that the property shown in the \uMVenv environment generalizes to environments with finite returns and non-vanishing transitions.
Consider \cref{fig:uMVEnv_param}, an environment similar to the \uMVenv environment but where the worst- and best case outcomes are set to $\pm 1$ instead of $\pm \infty$.
In this case, the expected value when measuring is given by $(1-p) - c$, meaning measuring is optimal for $(1-p) \geq c$.
Thus, for any $p_{\text{max}} < c$, measuring will be sub-optimal, even though it would yield positive returns for any $(1-p) \geq c$.

\subsection{Proofing \cref{thm:CRbound}}

In this section, we provide proof for \cref{thm:CRbound}, which has been left out of the paper due to space constraints.
More precisely, we state and prove a more general theorem from which \cref{thm:CRbound} follows.

\begin{theorem}
\label{thm:MoreMeasuringBound}
    Given an \RACNO $\mathcal{M}$ and a set of belief states $\mathcal{B}$.
    Let $\pi$ be any policy such that $b \in \mathcal{B} \implies \exists a: \pi(b) = \langle a, 0 \rangle$, and $\pi'$ a policy defined as follows:
    \begin{equation}
    \label{eq:defPiMoreMeasuring}
        \pi'(b) = 
        \begin{cases}
            \langle a, 1 \rangle & \text{ if } b \in \mathcal{B} \text{ and } \pi(b) = \langle a, 0 \rangle \\
            \pi(b) & \text{ otherwise }
        \end{cases}
    \end{equation}
    Furthermore, let $N_\mathcal{B}(\pi, b_0)$ denote a (possibly infinite) upper limit for the expected number of visits of any belief in $b \in B$ when following a policy $\pi$ from any (initial) belief $b_0$, defined as follows:
    \begin{equation}
    \label{eq:BmDefinition}
        N_\mathcal{B}(\pi, b_0) =
        \begin{cases}
            \lceil \sum_{s \in S} b'_R(s | b_0, \pi(b_0)) N_\mathcal{B} (\pi, s) \rceil + 1 \\ \hspace{15pt} \text{ if } b_0\in \mathcal{B} \text { and } \exists a:  \pi(b_0) = \langle a, 1 \rangle \\
            \lceil \sum_{s \in S} b'_R(s | b_0, \pi(b_0)) N_\mathcal{B} (\pi, s) \rceil \\ \hspace{15pt} \text{ if } b_0 \notin \mathcal{B} \text { and } \exists a: \pi(b_0) = \langle a, 1 \rangle \\
            N_\mathcal{B}(\pi, b'_R(b_0, \pi(b_0))) + 1 \\ \hspace{15pt} \text{ if } b_0 \in \mathcal{B} \text { and } \exists a: \pi(b_0) = \langle a, 0 \rangle\\
            N_\mathcal{B}(\pi, b'_R(b_0, \pi(b_0))) \\ \hspace{15pt} \text{ if } b_0 \notin \mathcal{B} \text { and } \exists a: \pi(b_0) = \langle a, 0 \rangle.
        \end{cases}
    \end{equation}
    Then, the following holds:
    \begin{equation}
        \forall b: V(\pi,b) - V(\pi',b) \leq \sum_{n=0}^{N_\mathcal{B}(\pi',b)} \gamma^n c.
    \end{equation}

\end{theorem}

\begin{corollary}
    Since measurement lenient policies can be defined from their robust counterparts using \cref{eq:defPiMoreMeasuring}, their performance loss follows the same bound.
\end{corollary}

\begin{corollary}
    For any environment where $\mathcal{B}$ or $N_\mathcal{B}$ are not known, we can use the over-estimation given in \cref{thm:CRbound}:
    \begin{equation}
        \forall b: V(\pi,b) - V(\pi',b) \leq \sum_{n=0}^{\infty} \gamma^n c
    \end{equation}
\end{corollary}

\begin{proof}
    Let $V^H$ and $N_\mathcal{B}^H \leq H$ denote $V$ and $N_\mathcal{B}$ for some horizon $H$.
    Then, the following is equivalent to our theorem:
    
    \begin{equation}
        \forall b: \lim_{H \rightarrow \infty} V^H(\pi,b) - V^H(\pi',b) \leq \sum_{n=0}^{ N^H_\mathcal{B}(\pi',b)} \gamma^n c
    \end{equation}
    
    We show this equation holds via induction over $H$.
    As a base case, we note the equation trivially holds for $H=1$, in which case the difference is simply given by $c$.

    For $H>1$, we first note that the equation trivially holds for beliefs where both policies pick the same action pair.
    Thus, we only need to prove the equation holds for beliefs $b \in \mathcal{B}$, which are the only beliefs where both policies pick different measurements.
    For any such belief $b$, denote the control action picked by both policies as $a$, then the difference in finite-horizon value functions is given as:

    \begin{equation}
    \begin{aligned}
    \label{eq:MoreMeasurementProofEq1}
        & V^H (\pi, b) - V^H(\pi', b) \\
        = & \Big(R(b,a) + \gamma V^{H-1}(\pi, b'_R \big( b, \langle a, 0 \rangle \big) \Big) - \\
          & \Big( R(b,a) - c + \gamma V^{H-1}(\pi', b'_R \big( b, \langle a, 1 \rangle \big) \Big) \\
        = & c {+} \gamma \Big( V^{H{-}1}\big(\pi, b'_R ( b, \langle a, 0 \rangle ) \big) {-} V^{H{-}1}(\pi', b'_R \big( b, \langle a, 1 \rangle \big) \Big)
    \end{aligned}
    \end{equation}

    To simplify this, we use the following general inequality:
    
    \begin{equation}
        \forall b,a,\pi: V(\pi, b'_R(b, \langle a, 0 \rangle )) \leq V(\pi, b'_R(b, \langle a, 1 \rangle )) 
    \end{equation}
    
    Using this, we replace $b'_R( b, \langle a, 0 \rangle )$ with $b'_R( b, \langle a, 1 \rangle )$ in the first value function of \cref{eq:MoreMeasurementProofEq1}.    
    Next, we can re-write our value function as a sum over states, as follows:

    \begin{equation}
    \begin{aligned}
    \label{eq:MoreMeasurementProofEq3}
        & V^H(\pi, b) - V^H(\pi', b) \\
        \leq & c + \gamma \sum_{s \in S} b_R'(s' | b, \langle a, 1 \rangle) \Big( V^{H-1}(\pi, s) - V^{H-1}(\pi', s) \Big)
        \\
        \leq & c + \gamma \Big( \sum_{s \in S} b_R'(s' | b, \langle a, 1 \rangle) \sum^{ N_\mathcal{B}^H(\pi',s) }_{n=0} \gamma^n c \Big),
    \end{aligned}
    \end{equation}

    where we obtain the second line by using our induction hypothesis.
    Next, we try finding an upper bound for the bracketed term.
    From the definition of $N_\mathcal{B}$, we obtain the following constraint for measuring for beliefs in $\mathcal{B}$;
    
    \begin{equation}
        \sum_{s \in S} b_R'(s|b,\pi(b)) N^H_\mathcal{B}(\pi,s) \leq N_\mathcal{B}^H(\pi,b) - 1.
    \end{equation}

    For each state $s$, we notice that the contribution to this constraint grows quicker with $N^H_\mathcal{B}(\pi,s)$ than its contribution to the bracketed term in \cref{eq:MoreMeasurementProofEq3}.
    Thus, it achieves its maximum when $N^H_\mathcal{B}(\pi,s)$ is equal for all next states, or more precisely when $\forall s: N_\mathcal{B}^N(\pi',s) = N_\mathcal{B}^N(\pi',b) - 1$.
    We use this as an upper bound to \cref{eq:MoreMeasurementProofEq3}, in which case the sum over all states can be left out to give the following:

    \begin{equation}
    \begin{aligned}
        V^H(\pi, b) - V^H(\pi', b) & \leq c + \gamma \sum_{n=0}^{N_\mathcal{B}^N(\pi',b) - 1} \gamma^n c \\
        & \leq  \sum_{n=0}^{N_\mathcal{B}^N(\pi',b)} \gamma^n c .
    \end{aligned}
    \end{equation}

    This proves our induction step for $H$ and thus the theorem.       
\end{proof}

\end{document}